\documentclass[11pt]{article}

\usepackage{microtype}
\usepackage{graphicx}
\usepackage{subfigure}
\usepackage{booktabs} 

\usepackage[letterpaper, left=1in, right=1in, top=1in, bottom=1in]{geometry}

\usepackage[dvipsnames]{xcolor}
\usepackage[colorlinks=true, linkcolor=blue, citecolor=blue, urlcolor=black]{hyperref}
\usepackage{microtype}

\usepackage{algorithm}
\usepackage{algpseudocode}

\usepackage{natbib}
\bibpunct{(}{)}{;}{a}{,}{,}

\usepackage{amsthm}
\usepackage{mathtools}
\usepackage{amsmath}
\usepackage{bbm}
\usepackage{amsfonts}
\usepackage{amssymb}

\usepackage{MnSymbol} 

\usepackage{xpatch}

\newtheorem{theorem}{Theorem}
\newtheorem{remark}{Remark}

\newtheorem{definition}{Definition}
\newtheorem{proposition}{Proposition}
\theoremstyle{definition}  

\newtheorem{lemma}{Lemma}
\newtheorem{assumption}{Assumption}

\xpatchcmd{\proof}{\itshape}{\normalfont\proofnameformat}{}{}
\newcommand{\proofnameformat}{\bfseries}

\usepackage{prettyref}
\newcommand{\pref}[1]{\prettyref{#1}}
\newcommand{\pfref}[1]{Proof of \prettyref{#1}}
\newcommand{\savehyperref}[2]{\texorpdfstring{\hyperref[#1]{#2}}{#2}}
\newrefformat{eq}{\savehyperref{#1}{\textup{(\ref*{#1})}}}
\newrefformat{eqn}{\savehyperref{#1}{Equation~\ref*{#1}}}
\newrefformat{lem}{\savehyperref{#1}{Lemma~\ref*{#1}}}
\newrefformat{def}{\savehyperref{#1}{Definition~\ref*{#1}}}
\newrefformat{fact}{\savehyperref{#1}{Fact~\ref*{#1}}}
\newrefformat{line}{\savehyperref{#1}{line~\ref*{#1}}}
\newrefformat{thm}{\savehyperref{#1}{Theorem~\ref*{#1}}}
\newrefformat{cor}{\savehyperref{#1}{Corollary~\ref*{#1}}}
\newrefformat{sec}{\savehyperref{#1}{Section~\ref*{#1}}}
\newrefformat{app}{\savehyperref{#1}{Appendix~\ref*{#1}}}
\newrefformat{ass}{\savehyperref{#1}{Assumption~\ref*{#1}}}
\newrefformat{ex}{\savehyperref{#1}{Example~\ref*{#1}}}
\newrefformat{fig}{\savehyperref{#1}{Figure~\ref*{#1}}}
\newrefformat{alg}{\savehyperref{#1}{Algorithm~\ref*{#1}}}
\newrefformat{rem}{\savehyperref{#1}{Remark~\ref*{#1}}}
\newrefformat{conj}{\savehyperref{#1}{Conjecture~\ref*{#1}}}
\newrefformat{prop}{\savehyperref{#1}{Proposition~\ref*{#1}}}
\newrefformat{proto}{\savehyperref{#1}{Protocol~\ref*{#1}}}
\newrefformat{prob}{\savehyperref{#1}{Problem~\ref*{#1}}}
\newrefformat{claim}{\savehyperref{#1}{Claim~\ref*{#1}}}

\DeclarePairedDelimiter{\abs}{\lvert}{\rvert} %
\DeclarePairedDelimiter{\brk}{[}{]}
\DeclarePairedDelimiter{\crl}{\{}{\}}
\DeclarePairedDelimiter{\prn}{(}{)}
\DeclarePairedDelimiter{\nrm}{\|}{\|}
\DeclarePairedDelimiter{\tri}{\langle}{\rangle}

\let\Pr\undefined

\DeclareMathOperator{\En}{\mathbb{E}}

\DeclareMathOperator{\Pr}{Pr}

\DeclareMathOperator*{\argmin}{arg\,min} 
\DeclareMathOperator*{\argmax}{arg\,max}

\newcommand{\ls}{\ell}
\newcommand{\ind}{\mathbbm{1}}    
\newcommand{\pmo}{\crl*{\pm{}1}}
\newcommand{\eps}{\epsilon}
\newcommand{\veps}{\varepsilon}

\newcommand{\ldef}{\vcentcolon=}
\newcommand{\rdef}{=\vcentcolon}

\newcommand{\wt}[1]{\widetilde{#1}}
\newcommand{\wh}[1]{\widehat{#1}}

\def\ddefloop#1{\ifx\ddefloop#1\else\ddef{#1}\expandafter\ddefloop\fi}
\def\ddef#1{\expandafter\def\csname bb#1\endcsname{\ensuremath{\mathbb{#1}}}}
\ddefloop ABCDEFGHIJKLMNOPQRSTUVWXYZ\ddefloop
\def\ddefloop#1{\ifx\ddefloop#1\else\ddef{#1}\expandafter\ddefloop\fi}
\def\ddef#1{\expandafter\def\csname b#1\endcsname{\ensuremath{\mathbf{#1}}}}
\ddefloop ABCDEFGHIJKLMNOPQRSTUVWXYZ\ddefloop
\def\ddef#1{\expandafter\def\csname c#1\endcsname{\ensuremath{\mathcal{#1}}}}
\ddefloop ABCDEFGHIJKLMNOPQRSTUVWXYZ\ddefloop
\def\ddef#1{\expandafter\def\csname h#1\endcsname{\ensuremath{\widehat{#1}}}}
\ddefloop ABCDEFGHIJKLMNOPQRSTUVWXYZ\ddefloop
\def\ddef#1{\expandafter\def\csname hc#1\endcsname{\ensuremath{\widehat{\mathcal{#1}}}}}
\ddefloop ABCDEFGHIJKLMNOPQRSTUVWXYZ\ddefloop
\def\ddef#1{\expandafter\def\csname t#1\endcsname{\ensuremath{\widetilde{#1}}}}
\ddefloop ABCDEFGHIJKLMNOPQRSTUVWXYZ\ddefloop
\def\ddef#1{\expandafter\def\csname tc#1\endcsname{\ensuremath{\widetilde{\mathcal{#1}}}}}
\ddefloop ABCDEFGHIJKLMNOPQRSTUVWXYZ\ddefloop

\newcommand{\Holder}{H{\"o}lder}

\let\wt\undefined

\newcommand{\wt}[1]{\widetilde{#1}}

\newcommand{\diag}{\textrm{diag}}

\newcommand{\sgn}{\textnormal{sgn}}

\newcommand{\grad}{\nabla}

\newcommand{\trn}{\intercal}

\renewcommand{\trn}{\dagger}
\newcommand{\Tr}{\mathbf{Tr}}

\usepackage{times}

\usepackage{bold-extra}

\usepackage{enumitem}

\usepackage{nicefrac}

\usepackage{algorithm}
\usepackage{mdframed}
\theoremstyle{definition}
\newmdtheoremenv[innertopmargin=5pt,linewidth=.3mm]{framedalgorithm}[algorithm]{Algorithm}

\newcommand{\midsem}{\,;\,}
\newcommand{\dmid}{\|}

\newcommand{\reg}{\cR}
\newcommand{\breg}{D_{\reg}}

\newcommand{\algcomment}[1]{\textcolor{Blue}{\footnotesize{\texttt{\textbf{// #1}}}}}

\newcommand{\poprisk}{L_{\cD}}
\newcommand{\loss}{\ls}

\newcommand{\halg}{\hat{h}}
\newcommand{\walg}{\wh{w}}

\newcommand{\poly}{\mathrm{poly}}

\newcommand{\maurey}{Q^{s}}

\renewcommand{\trn}{\top}
\renewcommand{\Tr}{\mathsf{tr}}

\newcommand{\comp}{c}

\newcommand{\sigmah}{\wh{\sigma}}

\newcommand{\approxleq}{\lesssim}

\renewcommand{\paragraph}[1]{\par\textbf{#1}\hspace{5pt}}

\newcommand{\new}[1]{#1}

\title{Distributed Learning with Sublinear Communication}
\author{
Jayadev Acharya\\
Cornell University\\
{\small acharya@cornell.edu}
\and
Christopher De Sa\\
Cornell University\\
{\small cdesa@cs.cornell.edu}
\and
Dylan J. Foster\\
MIT\\
{\small dylanf@mit.edu}
\and
Karthik Sridharan\\
Cornell University\\
{\small sridharan@cs.cornell.edu}
}
\date{}

\begin{document}
\maketitle




%



\begin{abstract}
In distributed statistical learning, $N$ samples are split across $m$ machines and a learner wishes to use minimal communication to learn as well as if the examples were on a single machine. 
This model has received substantial interest in machine learning due
to its scalability and potential for parallel speedup. However, in high-dimensional settings, where the number examples is smaller than the number of features (``dimension''), the speedup afforded by distributed learning may be overshadowed by the cost of communicating a single example. This paper investigates
 the following question: When is it possible to learn a $d$-dimensional model in the distributed setting with total communication sublinear in $d$?

Starting with a negative result, we observe that for learning $\ls_1$-bounded or sparse linear models, no algorithm can obtain optimal error
until communication is linear in dimension. Our main result is that
that by slightly relaxing the standard boundedness assumptions for
linear models, we can obtain distributed algorithms that enjoy optimal
error with communication \emph{logarithmic} in dimension. This result is based on a family of algorithms that combine mirror descent
with randomized sparsification/quantization of iterates, and extends to
the general stochastic convex optimization model.

\end{abstract}

\section{Introduction}
\label{sec:intro}

In statistical learning, a learner receives examples $z_1,\ldots,z_{N}$ i.i.d. from an unknown distribution $\cD$. Their  goal is to output a hypothesis $\halg\in\cH$ that minimizes the prediction error $\poprisk(h)\ldef{}\En_{z\sim{}\cD}\loss(h,z)$, and in particular to guarantee that \emph{excess risk} of the learner is small, i.e.
\begin{equation}
\label{eq:excess_risk}
\poprisk(\halg) - \inf_{h\in\cH}\poprisk(h)\leq\veps(\cH,N),
\end{equation}
where $\veps(\cH,N)$ is a decreasing function of $N$. This paper focuses on \emph{distributed} statistical learning. Here, the $N$ examples are split evenly across $m$ machines, with $n\ldef{}N/m$ examples per machine, and the learner wishes to achieve an excess risk guarantee such as \pref{eq:excess_risk} with minimal overhead in computation or communication.

Distributed learning has been the subject of extensive investigation due to its scalability for processing massive data: We may wish to efficiently process datasets that are spread across multiple data-centers, or we may want to distribute data across multiple machines to allow for parallelization of learning procedures. The question of parallelizing computation via distributed learning is a well-explored problem \citep{bekkerman2011scaling,recht2011hogwild,dekel2012optimal,chaturapruek2015asynchronous}. However, one drawback that limits the practical viability of these approaches is that the communication cost amongst machines may overshadow gains in parallel speedup \citep{bijral2016data}. Indeed, for high-dimensional statistical inference tasks where $N$ \new{could be} much smaller than the dimension $d$, or in modern deep learning models where \new{the} number of model parameters exceeds the number of examples (e.g. \cite{he2016deep}), communicating a single gradient or sending the raw model parameters between machines constitutes a significant overhead. 

Algorithms with reduced communication complexity in distributed learning have received significant recent development \citep{seide20141,alistarh2017qsgd,zhang2017zipml,suresh2017distributed, bernstein2018signsgd,tang2018communication}, but typical results here take as a given that when gradients or examples live in $d$ dimensions, communication will scale as $\Omega(d)$. Our goal is to revisit this tacit assumption and understand when it can be relaxed. We explore the question of \emph{sublinear communication}:

{\centering
\textit{Suppose a hypothesis class $\cH$ has $d$ parameters. When is it possible to achieve optimal excess risk for $\cH$ in the distributed setting using $o(d)$ communication?}

}

\subsection{Sublinear Communication for Linear Models?}
In this paper we focus on linear models, which are a special case of the general learning setup \pref{eq:excess_risk}. We restrict to linear hypotheses of the form $h_{w}(x) = \tri*{w,x}$ where $w,x\in\bbR^{d}$ and write $\ls(h_{w},z)=\phi(\tri*{w,x},y)$, where $\phi(\cdot,y)$ is a fixed link function and $z=(x,y)$. We overload notation slightly and write 
\begin{equation}
\label{eq:linear_model}
\poprisk(w)=\En_{(x,y)\sim\cD}\phi(\tri*{w,x},y).
\end{equation} The formulation captures standard learning tasks such as square loss regression, where $\phi(\tri*{w,x},y)=\prn*{\tri*{w,x}-y}^{2}$, logistic regression, where $\phi(\tri*{w,x},y)=\log\prn*{1+e^{-y\tri*{w,x}}}$, and classification with surrogate losses such as the hinge loss, where $\phi(\tri*{w,x},y)=\max\crl*{1-\tri*{w,x}\cdot{}y,0}$.

Our results concern the communication complexity of learning for linear models in the \emph{$\ls_p/\ls_q$-bounded setup}: weights belong to $\cW_{p}\ldef\crl*{w\in\bbR^{d}\mid\nrm*{w}_{p}\leq{}B_{p}}$ and feature vectors belong to $\cX_{q}\ldef\crl*{x\in\bbR^{d}\mid\nrm*{x}_{q}\leq{}R_{q}}$.\footnote{Recall the definition of the $\ls_p$ norm: $\nrm*{w}_{p}=\prn*{\sum_{i=1}^{d}\abs*{w_i}^{p}}^{1/p}$.} This setting is a natural starting point to investigate sublinear-communication distributed learning because \emph{learning is possible even when $N\ll{}d$.} 

Consider the case where $p$ and $q$ are dual, i.e. $\frac{1}{p}+\frac{1}{q}=1$, and where $\phi$ is $1$-Lipschitz. Here it is well known \citep{zhang2002covering,kakade2009complexity} that whenever $q\geq{}2$, the optimal sample complexity for learning, which is achieved by choosing the learner's weights $\wh{w}$ using empirical risk minimization (ERM), is
\begin{equation}
\label{eq:lp_risk}
\poprisk(\wh{w}) - \inf_{w\in\cW_{p}}\poprisk(w) = \Theta\prn*{
\sqrt{\frac{B_{p}^{2}R_{q}^{2}C_{q}}{N}}
},
\end{equation}
where $C_{q}=q-1$ for finite $q$ and $C_{\infty}=\log{}d$, or in other words
\begin{equation}
\label{eq:l1_risk}
\poprisk(\wh{w}) - \inf_{w\in\cW_{1}}\poprisk(w) = \Theta\prn*{
\sqrt{\frac{B_{1}^{2}R_{\infty}^{2}\log{}d}{N}}
}.
\end{equation}
We see that when $q<\infty$ the excess risk for the dual $\ls_{p}/\ls_{q}$
setting is \emph{independent of dimension} so long as the norm bounds
$B_p$ and $R_q$ are held constant, and that even in the
$\ls_1/\ls_{\infty}$ case there is only a mild logarithmic dependence. Hence, we can get nontrivial excess risk even when the number of examples $N$ is arbitrarily small compared to the dimension
$d$. This raises the intriguing question: \textit{Given that we can obtain nontrivial excess risk when $N \ll d$, can we obtain nontrivial excess risk when \emph{communication} is sublinear in $d$?}

To be precise, we would like to develop algorithms that achieve
\pref{eq:lp_risk}/\pref{eq:l1_risk} with total
bits of communication $\mathrm{poly}(N,m,\log{}d)$, permitting also
$\poly(B_p,R_q)$ dependence. The prospect of such a guarantee is exciting because---in light of the discussion above---as this would imply that we can obtain nontrivial excess risk with fewer bits of total communication than are required to naively send a \emph{single feature vector}.

\subsection{Contributions}
We provide new communication-efficient distributed learning algorithms and lower bounds for $\ls_{p}/\ls_{q}$-bounded linear models, and more broadly, stochastic convex optimization. We make the following observations:
\begin{itemize}[topsep=0pt]
\item For $\ls_2/\ls_2$-bounded linear models, sublinear communication \emph{is} achievable, and is obtained by using a derandomized Johnson-Lindenstrauss transform to compress examples and weights.
\item For $\ls_1/\ls_{\infty}$-bounded linear models, no distributed algorithm can obtain optimal excess
  risk until communication is \emph{linear} in dimension.  
\end{itemize}
These observations lead to our main result. We show that by relaxing the $\ls_1/\ls_{\infty}$-boundedness assumption and instead learning $\ls_{1}/\ls_{q}$-bounded models for a constant $q < \infty$, one unlocks a plethora of new algorithmic tools for sublinear distributed learning:
  \begin{enumerate}[topsep=0pt]
\item  We give an algorithm with optimal rates matching \pref{eq:lp_risk}, with
  communication $\mathrm{poly}(N,m^{q},\log{}d)$.  
  \item We extend the sublinear-communication algorithm to give refined guarantees, including instance-dependent \emph{small loss} bounds for smooth losses, \emph{fast rates} for strongly convex losses, and optimal rates for matrix learning problems.
\end{enumerate}
Our main algorithm is a distributed version of mirror descent that
uses randomized sparsification of weight vectors to reduce communication. Beyond learning in
linear models, the algorithm enjoys guarantees for the more general distributed
stochastic convex optimization model.%

To elaborate on the fast rates mentioned above, another important case where learning is possible when $N\ll{}d$ is the sparse high-dimensional linear model setup central to compressed sensing and statistics. Here, the standard result is that when $\phi$ is strongly convex and the benchmark class consists of $k$-sparse linear predictors,
i.e. $\cW_{0}\ldef\crl*{w\in\bbR^{d}\mid\nrm*{w}_{0}\leq{}k}$, one can guarantee
\begin{equation}
\label{eq:sparse_risk}
\poprisk(\wh{w}) - \inf_{w\in\cW_{0}}\poprisk(w) = \Theta\prn*{
  \frac{k\log{}(d/k)}{N}
}.
\end{equation}
With $\ls_{\infty}$-bounded features, no algorithm can obtain optimal excess risk for this setting until communication is linear in dimension, even under compressed sensing-style assumptions. When features are $\ls_{q}$-bounded however, our general machinery gives optimal fast rates matching \pref{eq:sparse_risk} under Lasso-style assumptions, with communication $\mathrm{poly}(N^{q},\log{}d)$.

The remainder of the paper is organized as follows. In \pref{sec:overview}
we develop basic upper and lower bounds for the $\ls_2/\ls_2$ and $\ls_1/\ls_{\infty}$-bounded settings. Then in
\pref{sec:upper_bounds} we shift to the $\ls_{1}/\ls_{q}$-bounded setting, where we introduce the family of sparsified mirror descent algorithms that leads to our main results and sketch the analysis. 

\subsection{Related Work}
Much of the work in algorithm design for distributed learning and optimization does not explicitly consider the number of bits used in communication per messages, and instead tries to make communication efficient via other means, such as decreasing the communication frequency or making learning robust to network disruptions \citep{duchi2012dual,zhang2012communication}. Other work reduces the number of bits of communication, but still requires that this number be linear in the dimension $d$. One particularly successful line of work in this vein is low-precision training, which represents the numbers used for communication and elsewhere within the algorithm using few bits \citep{alistarh2017qsgd,zhang2017zipml,seide20141,bernstein2018signsgd,tang2018communication,stich2018sparsified,alistarh2018convergence}.
Although low-precision methods have seen great success and adoption in neural network training and inference, low-precision methods are fundamentally limited to use bits proportional to $d$; once they go down to one bit per number there is no additional benefit from decreasing the precision.
Some work in this space tries to use sparsification to further decrease the communication cost of learning, either on its own or in combination with a low-precision representation for numbers \citep{alistarh2017qsgd,wangni2018gradient,wang2018atomo}. While the majority of these works apply low-precision and sparsification to gradients, a number of recent works apply sparsification to model parameters \citep{tang2018communication,stich2018sparsified,alistarh2018convergence}; We also adopt this approach. The idea of sparsifying weights is not new \citep{shalev2010trading}, but our work is the first to provably give communication logarithmic in dimension. To achieve this, our assumptions and analysis are quite a bit different from the results mentioned above, and we crucially use mirror descent, departing from the gradient descent approaches in \cite{tang2018communication,stich2018sparsified,alistarh2018convergence}.

Lower bounds on the accuracy of learning procedures with limited memory and communication have been explored in several settings, including mean estimation, sparse regression, learning parities, detecting correlations, and independence testing \citep{shamir2014fundamental,duchi2014optimality,garg2014communication,steinhardt2015minimax,braverman2016communication,steinhardt2016memory,acharya2018distributed, acharya2018inference, raz2018fast,han2018geometric,sahasranand2018extra, dagan2018detecting,dagan2019space}. In particular, the results of \cite{steinhardt2015minimax} and \cite{braverman2016communication} imply that optimal algorithms for distributed sparse regression need communication much larger than the sparsity level under various assumptions on the number of machines and communication protocol.


\section{Linear Models: Basic Results}
\label{sec:overview}

In this section we develop basic upper and lower bounds for
communication in $\ls_2/\ls_2$- and $\ls_1/\ls_{\infty}$-bounded linear
models. Our goal is to highlight some of the
counterintuitive ways in which the interaction between the geometry of
the weight vectors and feature vectors influences the communication
required for distributed learning. In particular, we wish to
underscore that the communication complexity of distributed learning
and the statistical complexity of centralized learning do not in
general coincide, and to motivate the $\ls_1/\ls_q$-boundedness
assumption under which we derive communication-efficient algorithms in \pref{sec:upper_bounds}.

\subsection{Preliminaries}
We formulate our results in a distributed communication model
following \citet{shamir2014fundamental}. Recalling that $n=N/m$, the model is as follows.
\begin{itemize}[topsep=0pt]
\item For machine $i=1,\ldots,m$:
\begin{itemize}[topsep=0pt]
\item Receive $n$ i.i.d. examples $S_i\ldef{}z_{1}^{i},\ldots,z_{n}^{i}$.
\item Compute message $W_i=f_i(S_{i}\midsem{}W_{1},\ldots,W_{i-1})$, where $W_{i}$ is at most $b_i$ bits.
\end{itemize}
\item Return $W=f(W_1,\ldots,W_m)$.
\end{itemize}
We refer to $\sum_{i=1}^{m}b_i$ as the \emph{total communication}, and we refer to any protocol with $b_i\leq{}b\;\forall{}i$ as a \emph{$(b,n,m)$ protocol}.
As a special case, this model captures a serial distributed learning setting where machines proceed
one after another: Each machine does some computation on their
data $z_1^{i},\ldots,z_n^{i}$ and previous messages
$W_1,\ldots,W_{i-1}$, then broadcasts their own message $W_i$ to all
subsequent machines, and the final model in \pref{eq:excess_risk} is
computed from $W$, either on machine $m$ or on a central server. The
model also captures protocols in which each machine independently computes a local
estimator and sends it to a central server, which aggregates the local
estimators to produce a final estimator \citep{zhang2012communication}. All of our upper bounds have the serial structure above, and our lower bounds apply to any $(b,n,m)$ protocol.

\subsection{$\ls_2/\ls_2$-Bounded Models}
\label{sec:l2l2}
In the $\ls_2/\ls_2$-bounded setting, we can \new{achieve sample optimal learning} with sublinear communication by using dimensionality
reduction. The idea is to project examples into
$k=\tilde{O}(N)$ dimensions using the
Johnson-Lindenstrauss transform, then perform a naive distributed
implementation of any standard learning algorithm in the projected
space. Here we implement the approach using stochastic gradient descent.

The first machine picks a JL matrix $A\in\bbR^{k\times{}d}$ and communicates the identity
of the matrix to the other $m-1$ machines. The JL matrix is chosen 
using the derandomized sparse JL transform of \citet{deJL},
and its identity can be communicated by sending the
random seed, which takes $O(\log(k/\delta) \cdot
\log{}d)$ bits for confidence parameter $\delta$. The dimension $k$ and
parameter $\delta$ are chosen as a function of $N$.

Now, each machine uses the matrix $A$ to project its features down to
$k$ dimensions. Letting $x'_t=Ax_t$ denote the projected features, the
first machine starts with a $k$-dimensional weight vector $u_1 = 0$
and performs the online gradient descent update
\citep{zinkevich2003online,PLG} over its $n$ projected samples as: 
\begin{equation*}
u_t \gets u_{t-1} - \eta\nabla \phi(\tri{u_t,x'_t}, y_t),
\end{equation*}
where $\eta>0$ is the learning rate. Once the first machine has passed over all its samples, it broadcasts
the last iterate $u_{n+1}$ as well the average $\sum_{s=1}^{n} u_s$,
which takes $\tilde{O}(k)$ communication. The next machine machine
performs the same sequence of gradient updates on its own data using
$u_{n+1}$ as the initialization, then passes its final iterate and the
updated average to the next machine. This repeats until we arrive at
the $m$th machine. The $m$th machine computes the $k$-dimensional
vector $\wh{u} \ldef \frac{1}{N} \sum_{t=1}^N u_t$, and returns
$\wh{w} = A^\top \hat{u}$ as the solution.

\begin{theorem}
\label{thm:jl}
When $\phi$ is $L$-Lipschitz and $k=\Omega(N\log(dN))$, the strategy above guarantees that 
$$
\mathbb{E}_S \mathbb{E}_A\left[L_{\cD}(\wh{w})\right] - \inf_{w\in\cW_2} L_{\cD}(w)\le  O\left(\sqrt{\frac{L^{2}B_2^2R_2^2}{N}}\right),
$$ 
where $\En_{S}$ denotes expectation over samples and $\En_{A}$ denotes expectation over the algorithm's randomness. The total communication is $O(mN\log(dN)\log(LB_2R_2N)+m\log(dN)\log{}d)$ bits.
\end{theorem} 

\subsection{$\ls_1/\ls_{\infty}$-Bounded Models: Model Compression}
While the results for the $\ls_2/\ls_2$-bounded setting are encouraging, they are
not useful in the common situation where features are
dense. When features are $\ls_{\infty}$-bounded, Equation~\pref{eq:l1_risk} shows that one can obtain
nearly dimension-independent excess risk so long as they restrict to
$\ls_1$-bounded weights. This $\ls_1/\ls_{\infty}$-bounded setting is particularly important
because it captures the fundamental problem of learning from a finite hypothesis
class, or \emph{aggregation} \citep{tsybakov2003optimal}: Given a class $\cH$ of $\pmo$-valued
predictors with $\abs*{\cH}<\infty$ we can set $x= (h(z))_{h\in\cH}\in\bbR^{\abs*{\cH}}$, in
which case \pref{eq:l1_risk} turns into the familiar finite class bound
$\sqrt{\log\abs*{\cH}/N}$ \citep{shalev2014understanding}. Thus,
algorithms with communication sublinear in dimension for the $\ls_1/\ls_\infty$
setting would lead to positive results in the general setting \pref{eq:excess_risk}.

As first positive result in this direction, we observe that by using the
well-known technique of \emph{randomized sparsification} or 
\emph{Maurey sparsification}, we can compress models to require only
logarithmic communication while
preserving excess risk.\footnote{We refer to the method as
  Maurey sparsification in reference to Maurey's early use of the
  technique in Banach spaces \citep{pisier1980remarques}, which predates its
  long history in learning theory
  \citep{jones1992simple,barron1993universal,zhang2002covering}.}
The method is simple: Suppose we have a weight vector $w$ that lies on
the simplex $\Delta_{d}$. We sample $s$ elements of $\brk*{d}$ i.i.d. according to $w$ and
return the empirical distribution, which we will denote
$\maurey(w)$. The empirical distribution is always $s$-sparse and can
be communicated using at most $O(s\log{}(ed/s))$ bits when
$s\leq{}d$,\footnote{That $O(s\log{}(ed/s))$ bits rather than, e.g.,
  $O(s\log{}d)$ bits suffice is a consequence of the usual ``stars and
  bars'' counting argument. We expect one can bring the
  expected communication down further using an adaptive scheme such as
  Elias coding, as in \citet{alistarh2017qsgd}.}
and it follows from standard concentration tools that by taking $s$
large enough the empirical distribution will approximate the true
vector $w$ arbitrarily well. 

The following lemma shows that Maurey sparsification indeed provides a
dimension-independent approximation to the excess risk in the
$\ls_1/\ls_{\infty}$-bounded setting. It applies to a version of the
Maurey technique for general vectors, which is given in
\pref{alg:maurey_l1}.
\begin{lemma}
\label{lem:maurey_prelim}
Let $w\in\bbR^{d}$ be fixed and suppose features belong to
$\cX_{\infty}$. When $\phi$ is $L$-Lipschitz, \pref{alg:maurey_l1} guarantees that
\begin{align}
\En{}\poprisk(Q^{s}(w)) \leq{} \poprisk(w) + \prn*{\frac{2L^{2}R^{2}_{\infty}\nrm*{w}_{1}^{2}}{s}}^{1/2},
\end{align}
where the expectation is with respect to the algorithm's randomness. Furthermore, when $\phi$ is $\beta$-smooth\footnote{A scalar function
  is said to be $\beta$-smooth if it has $\beta$-Lipschitz first
  derivative.} \pref{alg:maurey_l1} guarantees:
\begin{align}
\En{}\poprisk(Q^{s}(w)) \leq{} \poprisk(w) + \frac{\beta{}R^2_{\infty}\nrm*{w}_{1}^2}{s}.
\end{align}
\end{lemma}
The number of bits required to communicate $\maurey(w)$, including
sending the scalar $\nrm*{w}_{1}$ up to numerical precision, is at most
$O(s\log{}(ed/s) + \log(LB_1R_{\infty}s))$. Thus, if any single machine is able to find an
estimator $\walg$ with good excess risk, they can communicate it to
any other machine while preserving the excess risk with sublinear
communication. In particular, to preserve the optimal excess risk guarantee
in \pref{eq:l1_risk} for a Lipschitz loss such as absolute or hinge, the total bits of communication required is only
$O(N + \log{}(LB_1R_{\infty}N))$, which is indeed sublinear in dimension! For
smooth losses (square, logistic), this improves further to only
$O(\sqrt{N\log{}(ed/N)} + \log{}(LB_1R_{\infty}N))$
bits.

\begin{figure}[H]
\begin{framedalgorithm}[Maurey Sparsification]\mbox{}\\
\textbf{Input}: Weight vector $w\in\bbR^{d}$. Sparsity level $s$.
\setlist{nolistsep}
\begin{itemize}
\item Define $p\in\Delta_{d}$ via $p_{i}\propto{}\abs*{w_i}$.
\item For $\tau=1,\ldots,s$:
\begin{itemize}
\item Sample index $i_{\tau}\sim{}p$.
\end{itemize}
\item Return $Q^{s}(w) \ldef \frac{\nrm*{w}_{1}}{s}\sum_{\tau=1}^{s}\sgn(w_{i_{\tau}})e_{i_{\tau}}$.
\end{itemize}
\label{alg:maurey_l1}
\end{framedalgorithm}
\end{figure}


\subsection{$\ls_1/\ls_{\infty}$-Bounded Models: Impossibility}
\label{sec:lower_bounds}
Alas, we have only shown that \emph{if} we happen to find a good
solution, we can send it using sublinear communication. If we have to start from scratch, is it possible to use Maurey sparsification to coordinate between all machines to find a good solution? 


Unfortunately, the answer is no: For the $\ls_1/\ls_{\infty}$ bounded setting,
in the extreme case where each machine has a single example, \emph{no algorithm} can obtain a risk bound matching \pref{eq:l1_risk} until the number of
bits $b$ allowed per machine is (nearly) linear in $d$.
\begin{theorem}
\label{thm:lower_bound_slow}
Consider the problem of learning with the linear loss in the $(b,1,N)$ model, where risk
is $L_{\cD}(w)=\En_{(x,y)\sim\cD}\brk*{-y\tri*{w,x}}$. 
Let the benchmark class be the $\ls_1$ ball $\cW_{1}$, where $B_1=1$. For any algorithm $\walg$ there
exists a distribution $\cD$ with $\nrm*{x}_{\infty}\leq{}1$ and
$\abs*{y}\leq{}1$ such that
\[
\Pr\prn*{\poprisk(\wh{w}) - \inf_{w\in\cW_1}\poprisk(w) \geq{} \tfrac{1}{16}\sqrt{\tfrac{d}{b}\cdot\tfrac{1}{N}}\wedge\tfrac{1}{2}} \geq{} \tfrac{1}{2}.
\]
\end{theorem}
The lower bound also extends to the case of multiple examples per machine, albeit with
a less sharp tradeoff.
\begin{proposition}
\label{prop:lb_multi_machine}
Let $m$, $n$, and $\veps>0$ be fixed. In the setting of
\pref{thm:lower_bound_slow}, any algorithm in the $(b,n,m)$ protocol
with $b\leq{}O(d^{1-\veps/2}/\sqrt{N})$ has excess risk at least
$\Omega(\sqrt{d^{\veps}/N})$ with constant probability. 
\end{proposition}
This lower bound follows almost immediately from reduction to the ``hide-and-seek'' problem of
\citet{shamir2014fundamental}. The weaker guarantee from \pref{prop:lb_multi_machine} is a
consequence of the fact that the lower bound for the hide-and-seek
problem from \citet{shamir2014fundamental} is weaker in the
multi-machine case. 

The value of  \pref{thm:lower_bound_slow} and
\pref{prop:lb_multi_machine} is to rule out the possibility of obtaining optimal excess risk with
communication polylogarithmic in $d$ in the $\ls_1/\ls_{\infty}$ setting, even when there are many
examples per machine. This motivates the results of the next section,
which show that for $\ls_1/\ls_q$-bounded models it is indeed possible
to get polylogarithmic communication for any value of $m$.

One might hope that it is possible to circumvent \pref{thm:lower_bound_slow} by making
compressed sensing-type assumptions, e.g. assuming that the vector $w^{\star}$
is sparse and that restricted eigenvalue or a similar property is
satisfied. Unfortunately, this is not the case.

\begin{proposition}
\label{prop:lower_bound_fast}
Consider square loss regression in the $(b,1,N)$ model. For any algorithm $\wh{w}$ there exists a distribution $\cD$ with the following properties:
\begin{itemize}[topsep=0pt]
\item $\nrm*{x}_{\infty}\leq{}1$ and $\abs*{y}\leq{}1$ with probability $1$.
\item $\Sigma\ldef\En\brk*{xx^{\trn}}=I$, so that the population risk
  is $1$-strongly convex, and in particular has restricted strong
  convexity constant $1$.
\item $w^{\star}\ldef\argmin_{w:\nrm*{w}_{1}\leq{}1}\poprisk(w)$ is $1$-sparse.
\item Until $b=\Omega(d)$, $
\Pr\prn*{\poprisk(\wh{w}) - \poprisk(w^{\star}) \geq{} \tfrac{1}{256}\prn*{\tfrac{d}{b}\cdot\tfrac{1}{N}}\wedge\tfrac{1}{4}} \geq{} \tfrac{1}{2}$.
\end{itemize}
Moreover, any algorithm in the $(b,n,m)$ protocol
with $b\leq{}O(d^{1-\veps/2}/\sqrt{N})$ has excess risk at least
$\Omega(d^{\veps}/N)$ with constant probability.
\end{proposition}
That $\Omega(d)$ communication is required to obtain optimal excess
risk for $m=N$ was proven in \cite{steinhardt2015minimax}. The lower
bound for general $m$ is important here because it serves as a converse
to the algorithmic results we develop for sparse regression in
\pref{sec:upper_bounds}. It follows by reduction
to hide-and-seek.\footnote{\cite{braverman2016communication} also prove a communication lower bound for sparse regression. Their lower bound applies for all values of $m$ and for more sophisticated interactive protocols, but does not rule out the possibility of $\mathrm{poly}(N,m,\log{}d)$ communication.}


The lower bound for sparse linear
  models does not rule out that sublinear learning is possible using
  additional statistical assumptions, e.g. that there are many
  examples on each machine \emph{and} support recovery is possible. See \pref{app:lower_bounds} for detailed discussion.

\section{Sparsified Mirror Descent}
\label{sec:upper_bounds}

We now deliver on the promise outlined in the introduction and give
new algorithms with logarithmic communication under an assumption we
call \emph{$\ls_1/\ls_{q}$-boundness}. The model for which we derive
algorithms in this section is more general than the linear model setup
\pref{eq:linear_model} to which our lower bounds apply. We consider
problems of the form
\begin{equation}
\label{eq:stochastic_opt}
\underset{w\in\cW}{\mathrm{minimize}} \quad\poprisk(w)\ldef{}\En_{z\sim{}\cD}\ls(w,z),
\end{equation}
where $\ls(\cdot,z)$ is convex, $\cW\subseteq{}\cW_1=\crl*{w\in\bbR^{d}\mid{}\nrm*{w}_{1}\leq{}B_1}$
is a convex constraint set, and subgradients $\partial\ls(w,z)$ are
assumed to belong to
$\cX_q=\crl*{x\in\bbR^{d}\mid{}\nrm*{x}_{q}\leq{}R_{q}}$. This setting
captures linear models with $\ls_1$-bounded weights and
$\ls_q$-bounded features as a special case, but is considerably more
general, since the loss can be any Lipschitz function of $w$.

We have already shown that one cannot expect sublinear-communication
algorithms for $\ls_{1}/\ls_{\infty}$-bounded models, and so the
$\ls_q$-boundedness of subgradients in \pref{eq:stochastic_opt} may be thought of as strengthening our assumption on the data generating process. That this is stronger follows from the elementary fact that $\nrm*{x}_{q}\geq{}\nrm*{x}_{\infty}$ for all $q$.
\paragraph{Statistical complexity and nontriviality.}
For the dual $\ls_1/\ls_{\infty}$ setup in \pref{eq:linear_model} the
optimal rate is $\Theta(\sqrt{\log{}d/N})$.  While our goal is to find
minimal assumptions that allow for distributed learning with sublinear
communication, the reader may wonder at this point whether we have
made the problem easier \emph{statistically} by moving to the
$\ls_1/\ls_q$ assumption. The answer is ``yes, but only slightly.''
When $q$ is constant the optimal rate for $\ls_1/\ls_q$-bounded models
is $\Theta\prn{\sqrt{1/N}}$,\footnote{The upper bound follows from
  \pref{eq:lp_risk} and the lower bound follows by reduction to the
  one-dimensional case.} and so the effect of this assumption is to shave off the $\log{}d$ factor that was present in \pref{eq:l1_risk}.

\subsection{Lipschitz Losses}
\label{sec:lipschitz}

Our main algorithm is called \emph{sparsified mirror descent}
(\pref{alg:mirror_slow}). The idea behind the algorithm is to run the online mirror descent algorithm
\citep{ben2001lectures,hazan2016introduction}
in serial across the machines and sparsify the iterates whenever we move from one machine to the next.

In a bit more detail, \pref{alg:mirror_slow} proceeds from machine to
machine sequentially. On each machine, the algorithm generates a
sequence of iterates $w_1^i,\ldots,w_{n}^i$ by doing a single pass
over the machine's $n$ examples $z_{1}^{i},\ldots,z_{n}^{i}$ using the
mirror descent update with regularizer
$\cR(w)=\frac{1}{2}\nrm*{w}_{p}^{2}$, where
$\frac{1}{p}+\frac{1}{q}=1$, and using stochastic gradients $\grad_{t}^{i}\in\partial{}\ls(w_{t}^{i},z_{t}^{i})$. After the last example is processed on machine $i$, we compress the last iterate using Maurey sparsification (\pref{alg:maurey_l1}) and send it to the next machine, where the process is repeated.

To formally describe the algorithm, we recall the definition of the \emph{Bregman divergence}. Given a convex regularization function $\cR:\bbR^{d}\to\bbR$, the Bregman divergence with respect to $\cR$ is defined \new{as}
\[
\breg(w\dmid{}w')=\reg(w) - \reg(w') - \tri*{\grad\reg(w'),w-w'}.
\]
For the $\ls_{1}/\ls_{q}$ setting we exclusively use the regularizer $\cR(w)=\frac{1}{2}\nrm*{w}_{p}^{2}$, where $\frac{1}{p}+\frac{1}{q}=1$. 

The main guarantee for \pref{alg:mirror_slow} is as follows.
\begin{theorem}
\label{thm:stochastic_mirror_slow}
Let $q\geq{}2$ be fixed. Suppose that subgradients belong to $\cX_{q}$
and that $\cW\subseteq\cW_1$. If we run \pref{alg:mirror_slow} with  $\eta=\frac{B_1}{R_q}\sqrt{\frac{1}{C_{q}N}}$ and initial point $\bar{w}=0$, then whenever $s=\Omega(m^{2(q-1)})$ and $s_{0}=\Omega(N^{\frac{q}{2}})$ the algorithm guarantees
\[
\En\brk*{\poprisk(\wh{w})} -\poprisk(w^{\star}) \leq{} O\prn*{
\sqrt{\frac{B_1^2R_q^2 C_{q}}{N}}
},
\]
where $C_{q}=q-1$ is a constant depending only on $q$.

The total number of bits sent by each machine---besides communicating
the final iterate $\wh{w}$---is at most
$O(m^{2(q-1)}\log(d/m)+\log(B_1R_qN))$, and so the total number of
bits communicated globally is at most \[
O\prn*{N^{\frac{q}{2}}\log(d/N)  + m^{2q-1}\log(d/m) +m\log(B_1R_qN)}.
\]
In the linear model setting \pref{eq:linear_model} with $1$-Lipschitz
loss $\phi$ it suffices to set $s_{0}=\Omega(N)$, so that the total
bits of communication is
\[
O(N\log(d/N)  + m^{2q-1}\log(d/m) +m\log(B_1R_qN)).
\]

\end{theorem}

We see that the communication required by sparsified mirror descent is
exponential in the norm parameter $q$.  This means that whenever $q$
is constant, the overall communication is polylogarithmic in
dimension.  It is helpful to interpret the bound when $q$ is allowed
to grow with dimension. An elementary property of $\ls_q$ norms is
that for $q=\log{}d$, $\nrm*{x}_{q}\approx\nrm*{x}_{\infty}$ up to a
multiplicative constant. In this case the communication from
\pref{thm:stochastic_mirror_slow} becomes polynomial in dimension,
which we know from \pref{sec:lower_bounds} is necessary.

The guarantee of \pref{alg:mirror_slow} extends beyond the statistical
learning model to the first-order stochastic convex optimization model, as well as the online convex optimization model.

\begin{figure}[t]
\begin{framedalgorithm}[Sparsified Mirror Descent]\mbox{}\\
\setlist{nolistsep}
\textbf{Input}: \begin{itemize}[leftmargin=*]
\item[] Constraint set $\cW$ with $\nrm*{w}_{1}\leq{}B_1$. \\
Gradient norm parameter $q\in[2,\infty)$.\\  Gradient $\ls_{q}$ norm bound $R_q$.\\
 Learning rate $\eta$,  Initial point $\bar{w}$, Sparsity $s,s_{0}\in\bbN$.
\end{itemize}
\vspace{.5em}
Define $p=\frac{q}{q-1}$ and $\cR(w)=\frac{1}{2}\nrm*{w-\bar{w}}_{p}^{2}$. \vspace{.5em}\\
For machine $i=1,\ldots,m$:
\begin{itemize}
\vspace{.5em}
\item Receive $\wh{w}^{i-1}$ from machine $i-1$ and set $w^{i}_{1}=\wh{w}^{i-1}$ (if machine $1$ set $w^{1}_{1}=\bar{w}$).
\vspace{.5em}
\item For $t=1,\ldots,n$: \algcomment{Mirror descent step.}
\begin{itemize}
\item Get gradient $\grad_{t}^{i}\in\partial{}\ls(w_{t}^{i}\midsem{}z_{t}^{i})$.
\vspace{.5em}
\item $\grad\reg(\theta^{i}_{t+1})\gets\grad\reg(w^{i}_{t})-\eta\grad^{i}_t$.
\vspace{.5em}
\item $w_{t+1}^{i}\gets\argmin_{w\in\cW}\breg(w\dmid\theta^{i}_{t+1})$.
\end{itemize}
\vspace{.5em}
\item Let $\wh{w}^{i}\gets{}Q^{s}(w^{i}_{n+1})$. \algcomment{Sparsification.}
\vspace{.5em}
\item Send $\wh{w}^{i}$ to machine $i+1$.
\end{itemize}
\vspace{.5em}
Sample $i\in\brk*{m}$, $t\in\brk*{n}$ uniformly at random and return $\wh{w}\ldef{}Q^{s_{0}}(w^{i}_t)$.
\label{alg:mirror_slow}
\end{framedalgorithm}
\end{figure}

\paragraph{Proof sketch.}
They basic premise behind the algorithm and analysis is that by using the same
learning rate across all machines, we can pretend as though we are
running a single instance of mirror descent on a centralized
machine. The key difference from the usual analysis is that we need to
bound the error incurred by sparsification between successive
machines. Here, the choice of the regularizer is crucial. A
fundamental property used in the analysis of mirror descent is
\emph{strong convexity} of the regularizer. In particular, to give convergence rates that do not depend on dimension (such as \pref{eq:lp_risk}) it is essential that the regularizer be $\Omega(1)$-strongly convex. Our regularizer $\cR$ indeed has this property.
\begin{proposition}[\citet{ball1994sharp}]
\label{prop:lp_strong_convexity}
For $p\in(1,2]$, $\cR$ is $(p-1)$-strongly convex with respect to
$\nrm*{\cdot}_{p}$. Equivalently,
$\breg(w\dmid{}w')\geq{}\frac{p-1}{2}\cdot\nrm*{w-w'}_{p}^{2}\quad\forall{}w,w'\in\bbR^{d}$.
\end{proposition}
On the other hand, to argue that sparsification has negligible impact
on convergence, our analysis leverages \emph{smoothness} of the
regularizer. Strong convexity and smoothness are at odds with each
other: It is well known that in infinite dimension, any norm that is
both strongly convex and smooth is isomorphic to a Hilbert space \citep{pisier2011martingales}. What makes our analysis work is that while the regularizer
$\cR$ is not smooth, it is \emph{\Holder{}-smooth} for any finite
$q$. This is sufficient to bound the approximation error from
sparsification. To argue that the excess risk achieved by mirror
descent with the $\ls_p$ regularizer $\cR$ is optimal, however, it is essential that the
gradients are $\ls_q$-bounded rather than $\ls_\infty$-bounded.

In more detail, the proof can be broken into three components:
\begin{itemize}[topsep=0pt]
\item \emph{Telescoping.} Mirror descent gives a regret bound that telescopes across all $m$ machines up to the error introduced by sparsification. To argue that we match the optimal centralized regret, all that is required is to bound $m$ error terms of the form \[\breg(w^{\star}\dmid{}\maurey(w_{n+1}^{i})) - \breg(w^{\star}\dmid{}w_{n+1}^{i}).\]
\item \emph{\Holder{}-smoothness.} We prove (\pref{thm:lp_bregman_lipschitz}) that the difference above is of order
\[
B_1\nrm*{\maurey(w_{n+1}^{i})-w_{n+1}^{i}}_{p} + B^{3-p}_1\nrm*{\maurey(w_{n+1}^{i})-w_{n+1}^{i}}_{\infty}^{p-1}.
\]
\item \emph{Maurey for $\ls_p$ norms.} We prove (\pref{thm:maurey_lp}) that $\nrm*{\maurey(w_{n+1}^{i})-w_{n+1}^{i}}_{p}\approxleq{}\prn*{\frac{1}{s}}^{1-1/p}$ and likewise that
$\nrm*{\maurey(w_{n+1}^{i})-w_{n+1}^{i}}_{\infty}\approxleq{}\prn*{\frac{1}{s}}^{1/2}$.
\end{itemize}
With a bit more work these inequalities yield \pref{thm:stochastic_mirror_slow}. We close this section with a few more notes about \pref{alg:mirror_slow} and its performance.

\begin{remark}
We can modify \pref{alg:mirror_slow} so that it enjoys a high-probability excess risk bound by changing the final step slightly. Instead of subsampling $(i,t)$ randomly and returning $\maurey(w_t^i)$, have each machine $i$ average all its iterates $w_{1}^{i},\ldots,w_{n}^{i}$, then sparsify the average and send it to the final machine, which averages the averaged iterates from all machines and returns $\wh{w}$ as the result.

There appears to be a tradeoff here: The communication of the high
probability algorithm is $\tilde{O}(m^{2q-1} + mN^{q})$, while
\pref{alg:mirror_slow} has communication $\wt{O}(m^{2q-1} + N^{q})$. 
We leave a comprehensive exploration of this tradeoff for future work.
\end{remark}

\begin{remark}
For the special case of $\ls_1/\ls_q$-bounded linear models, it is not hard to show that the
following strategy also leads to sublinear communication: Truncate
each feature vector to the top $\Theta(N^{q/2})$ coordinates, then send all
the truncated examples to a central server, which returns the
empirical risk minimizer. This strategy matches the risk of
\pref{thm:stochastic_mirror_slow} with total communication
$\tilde{O}(N^{q/2+1})$, but has two deficiencies. First, it scales as
  $N^{O(q)}$, which is always worse than $m^{O(q)}$. Second, it
  does not appear to extend to the general optimization setting.
\end{remark}

\subsection{Smooth Losses}
We can improve the statistical guarantee and total communication further in the case where $\poprisk$ is \emph{smooth} with respect to $\ls_{q}$ rather than just Lipschitz. We assume that $\ls$ has $\beta_q$-Lipschitz gradients, in the sense that for all $w,w'\in\cW_1$ for all $z$,
\[
\nrm*{\grad\ls(w,z)-\grad{}\ls(w',z)}_{q}\leq{}\beta_{q}\nrm*{w-w'}_{p},
\]
where $p$ is such that $\frac{1}{p}+\frac{1}{q}$.
\begin{theorem}
\label{thm:mirror_linear_slow_lstar}
Suppose in addition to the assumptions of \pref{thm:stochastic_mirror_slow} that $\ls(\cdot,z)$ is non-negative and has $\beta_q$-Lipschitz gradients with respect to $\ls_q$. Let $L^{\star} = \inf_{w\in\cW}\poprisk(w)$.  If we run \pref{alg:mirror_slow} with learning rate $\eta=\sqrt{\frac{B^{2}_1}{C_{q}\beta_qL^{\star}N}}\wedge{}\frac{1}{4C_q\beta_q}$ and $\bar{w}=0$  then, if $s=\Omega(m^{2(q-1)})$ and $s_0 =
\sqrt{\frac{\beta_q{}B^2_1N}{C_qL^{\star}}}\wedge\frac{N}{C_q}$, the algorithm guarantees
\[
\En\brk*{L_{\cD}(\wh{w})} - L^{\star} \leq{}
O\prn*{
\sqrt{\frac{C_q \beta_qB^{2}_1L^{\star}}{N}} + \frac{C_q\beta_q{}B^{2}_1}{N}
}.
\]

The total number of bits sent by each machine---besides communicating
the final iterate $\wh{w}$---is at most $O(m^{2(q-1)}\log(d/m))$, and
so the total number of bits communicated globally is at most
\[
O\prn*{\prn*{\sqrt{\frac{\beta_q{}B^2_1N}{C_qL^{\star}}}\wedge\frac{N}{C_q}}\log(d/N)
+ m^{2q-1}\log(d/m)
+ m\log(\beta_q{}B_1N)}.\]
\end{theorem}

Compared to the previous theorem, this result provides a so-called
``small-loss bound'' \citep{srebro2010smoothness}, with the main term
scaling with the optimal loss $L^{\star}$. The dependence on $N$ in
the communication cost can be as low as $O(\sqrt{N})$ depending on the value of $L^{\star}$. 

\subsection{Fast Rates under Restricted Strong Convexity}
So far all of the algorithmic results we have present scale as
$O(N^{-1/2})$. While this is optimal for generic Lipschitz losses, we
mentioned in \pref{sec:overview} that for strongly convex losses the
rate can be improved in a nearly-dimension independent fashion to
$O(N^{-1})$ for sparse high-dimensional linear models. As in the
generic lipschitz loss setting, we show that making the assumption of
$\ls_1/\ls_q$-boundness is sufficient to get statistically optimal
distributed algorithms with sublinear communication, thus providing a
way around the lower bounds for fast rates in
\pref{sec:lower_bounds}. 

The key assumption for the results in this section is that the population risk satisfies a form of \emph{restricted strong convexity} over $\cW$:
\begin{assumption}
\label{ass:rsc}
There is some constant $\gamma_q$ such that \[\forall w\in\cW,\quad \poprisk(w)-\poprisk(w^{\star})-\tri*{\grad{}\poprisk(w^{\star}),w-w^{\star}}\geq{} \frac{\gamma_q}{2}\nrm*{w-w^{\star}}_{p}^{2}.\]
\end{assumption}
In a moment we will show how to relate this property to the standard restricted eigenvalue property in high-dimensional statistics \citep{negahban2012unified} and apply it to sparse regression.

\begin{figure}[t]
\begin{framedalgorithm}[Sparsified Mirror Descent for Fast Rates]\mbox{}\\
\textbf{Input}: \vspace{-.5em}
\begin{itemize}[leftmargin=*]
\item[] Constraint set $\cW$ with $\nrm*{w}_{1}\leq{}B_1$. \\
Gradient norm parameter $q\in[2,\infty)$.\\  Gradient $\ls_{q}$ norm bound $R_q$.\\
RSC constant $\gamma_q$.
Constant $c>0$.
\end{itemize}\vspace{-.5em}
Let $\wh{w}_0=0$, $B_{k}=2^{-k/2}B$ and $N_{k+1}=C_{q}\cdot\prn*{\frac{4cR}{\gamma{}B_{k-1}}}^{2}$.\vspace{.3em}\\
Let $T = \max\crl*{T\mid\sum_{k=1}^{T}N_{k}\leq{}N}$.\vspace{.3em}\\
Let examples have order: $z^{1}_1,\ldots,z^{1}_{n},\ldots,z^{m}_1,\ldots,z^{m}_n$.\vspace{.5em}\\
For round $k=1,\ldots,T$:\vspace{-.5em}
\begin{itemize}[leftmargin=*]
\item[] \emph{Let $\wh{w}_{k}$ be the result of running \pref{alg:mirror_slow} on $N_{k}$ consecutive examples in the ordering above, \\with the following configuration:\vspace{-.5em}
\begin{enumerate}
\item The algorithm begins on the example immediately after the last one processed at round $k-1$. 
\item The algorithm uses parameters $B_1$, $R_q$, $s$, $s_0$, and $\eta$ as prescribed in \pref{prop:mirror_slow_initial}, with initialization $\bar{w}=\wh{w}_{k-1}$ and radius $\bar{B}=B_{k-1}$.
\end{enumerate}}
\end{itemize}\vspace{-.5em}
Return $\wh{w}_{T}$.
\label{alg:mirror_fast}
\end{framedalgorithm}
\end{figure}

Our main algorithm for strongly convex losses is
\pref{alg:mirror_fast}. The algorithm does not introduce any new
tricks for distributed learning over \pref{alg:mirror_slow}; rather,
it invokes \pref{alg:mirror_slow} repeatedly in an inner loop, relying
on these invocations to take care of communication. This reduction is based on techniques developed in
\cite{juditsky2010primal}, whereby restricted strong convexity is used
to establish that error decreases geometrically as a function of the
number of invocations to the sub-algorithm. We refer the reader to \pref{app:upper_bounds} for additional details.

The main guarantee for \pref{alg:mirror_fast} is as follows.
\begin{theorem} 
\label{thm:lipschitz_fast}
Suppose \pref{ass:rsc} holds, that subgradients belong to $\cX_{q}$ for $q\geq{}2$, and that $\cW\subset\cW_1$.
When the parameter $c>0$ is a sufficiently large absolute constant, \pref{alg:mirror_fast} guarantees that
\[
\En\brk*{\poprisk(\wh{w}_{T})}-\poprisk(w^{\star})\leq{} O\prn*{
\frac{C_{q}R_{q}^{2}}{\gamma_qN}
}.
\]
The total numbers of bits communicated is
\[
O\prn*{\prn*{
N^{2(q-1)}m^{2q-1}\prn*{\frac{\gamma_q^{2}B_q^{2}}{C_{q}R_q^{2}}}^{2(q-1)}
+
N^{q}\prn*{\frac{\gamma{}_qB_1}{C_q{}R_q}}^{q}
}\log{}d + m\log(B_1R_qN)}.
\]
Treating scale parameters as constant, the total communication simplifies to $O\prn*{N^{2q-2}m^{2q-1}\log{}d}$.  
\end{theorem}
Note that the communication in this theorem depends polynomially on the various scale parameters, which was not the case for \pref{thm:stochastic_mirror_slow}.
\paragraph{Application: Sparse Regression.}
As an application of \pref{alg:mirror_fast}, we consider the sparse regression setting \pref{eq:sparse_risk}, where $\poprisk(w)=\En_{x,y}\prn*{\tri*{w,x}-y}^{2}$. We assume $\nrm*{x}_{q}\leq{}R_{q}$ and $\abs*{y}\leq{}1$. We let $w^{\star}=\argmin_{w\in\cW_{1}}\poprisk(w)$, so $\nrm*{w^{\star}}_{1}\leq{}B_1$. We assume $w^{\star}$ is $k$-sparse, with support set $S\subset{}\brk*{d}$.

We invoke \pref{alg:mirror_fast} constraint set $\cW\ldef\crl*{w\in\bbR^{d}\mid \nrm*{w}_{1}\leq{}\nrm*{w^{\star}}_{1}}$ and let $\Sigma=\En\brk*{xx^{\trn}}$. Our bound depends on the restricted eigenvalue parameter:\  $\gamma\ldef{}\inf_{\nu\in\cW-w^{\star}\setminus\crl*{0}} \nrm*{\Sigma^{1/2}\nu}_{2}^{2}/\nrm*{\nu}_{2}^{2}.$
\begin{proposition}
\label{prop:sparse_regression}
\pref{alg:mirror_fast}, with constraint set $\cW$ and appropriate
choice of parameters, guarantees:
\[
\En\brk*{\poprisk(\wh{w}_{T})}-\poprisk(w^{\star})\leq{} O\prn*{
C_{q}B^{2}_{1}R^{2}_{q}\cdot\frac{k}{\gamma{}N}
}.
\]
Suppressing problem-dependent constants, total communication is of order $O((N^{2q-2}m^{2q-1}\log{}d) /k^{4q-4})$.
\end{proposition}
  
\subsection{Extension: Matrix Learning and Beyond}
\label{sec:extensions}
The basic idea behind sparsified mirror descent---that by assuming
$\ls_{q}$-boundedness one can get away with using a \Holder{}-smooth
regularizer that behaves well under sparsification---is not limited to
the $\ls_1/\ls_q$ setting. To extend the algorithm to more general
geometry, all that is required is the following: 
\begin{itemize}[topsep=0pt]
\item The constraint set $\cW$ can be written as the convex hull of a set of atoms $\cA$ that has sublinear bit complexity.
\item The data should be bounded in some norm $\nrm*{\cdot}$ such that the dual $\nrm*{\cdot}_{\star}$ admits a regularizer $\cR$ that is strongly convex and \Holder{}-smooth with respect to $\nrm*{\cdot}_{\star}$
\item $\nrm*{\cdot}_{\star}$ is preserved under sparsification. We remark in passing that this property and the previous one are closely related to the notions of type and cotype in Banach spaces \citep{pisier2011martingales}.
\end{itemize}
Here we deliver on this potential and sketch how to extend the results so far to \emph{matrix learning} problems where $\cW\subseteq\bbR^{d\times{}d}$ is a convex set of matrices. As in \pref{sec:lipschitz} we work with a generic Lipschitz loss $\poprisk(W)=\En_{z}\ls(W,z)$. Letting $\nrm*{W}_{S_p}=\Tr((WW^{\trn})^{\frac{p}{2}})$ denote the Schatten $p$-norm, we make the following spectral analogue of the $\ls_1/\ls_{q}$-boundedness assumption: $\cW\subseteq\cW_{S_1}\ldef{}\crl*{W\in\bbR^{d\times{}d}\mid\nrm*{W}_{S_1}\leq{}B_1}$ and subgradients $\partial\ls(\cdot,z)$ belong to $\cX_{S_q}\ldef\crl*{X\in\bbR^{d\times{}d}\mid\nrm*{X}_{S_q}\leq{}R_q}$, where $q\geq{}2$. Recall that $S_1$ and $S_{\infty}$ are the nuclear norm and spectral norm. The $S_1/S_{\infty}$ setup has many applications in learning \citep{hazan2012near}.

We make the following key changes to \pref{alg:mirror_slow}:
\begin{itemize}[topsep=0pt]
\item Use the Schatten regularizer $\cR(W)=\frac{1}{2}\nrm*{W}_{S_p}^{2}$.
\item Use the following spectral version of the Maurey operator
  $\maurey(W)$: Let $W$ have singular value decomposition
  $W=\sum_{i=1}^{d}\sigma_iu_{i}v_{i}^{\trn}$ with $\sigma_i\geq{}0$
  and define $P\in\Delta_{d}$ via $P_i\propto\sigma_i$.\footnote{We
    may assume $\sigma_i\geq{}0$ without loss of generality.} Sample $i_1,\ldots,i_s$ i.i.d. from $P$ and return $\maurey(W)=\frac{\nrm*{W}_{S_1}}{s}\sum_{\tau=1}^{s}u_{i_{\tau}}v_{i_{\tau}}^{\trn}$.
\item Encode and transmit $\maurey(W)$ as the sequence $(u_{i_1}, v_{i_1}),\ldots,(u_{i_s}, v_{i_s})$, plus the scalar $\nrm*{W}_{S_1}$. This takes $\tilde{O}(sd)$ bits.
\end{itemize}

\begin{proposition}
\label{prop:schatten}
Let $q\geq{}2$ be fixed, and suppose that subgradients belong to $\cX_{S_q}$ and that $\cW\subseteq\cW_{S_1}$. If we run the variant of \pref{alg:mirror_slow} described above with learning rate $\eta=\frac{B_1}{R_q}\sqrt{\frac{1}{C_{q}N}}$ and initial point $\bar{W}=0$, then whenever $s=\Omega(m^{2(q-1)})$ and $s_{0}=\Omega(N^{\frac{q}{2}})$, the algorithm guarantees
\[
\En\brk*{\poprisk(\wh{W})} -\inf_{W\in\cW}\poprisk(W) \leq{} O\prn*{
\sqrt{\frac{B_1^2R_q^2 C_{q}}{N}}
},
\]
where $C_{q}=q-1$.
The total number of bits communicated globally is at most $\tilde{O}(m^{2q-1}d + N^{\frac{q}{2}}d)$.
\end{proposition}
In the matrix setting, the number of bits required to naively send weights $W\in\bbR^{d\times{}d}$ or subgradients $\partial\ls(W,z)\in\bbR^{d\times{}d}$ is $O(d^{2})$. The communication required by our algorithm scales only as $\tilde{O}(d)$, so it is indeed sublinear.

The proof of \pref{prop:schatten} is sketched in \pref{app:upper_bounds}. The key idea is that because the Maurey operator $\maurey(W)$ is defined in the same basis as $W$, we can directly apply approximation bounds from the vector setting.


\section{Discussion}
\label{sec:discussion}
We hope our work will lead to further development of algorithms with sublinear communication. A few immediate questions:
\begin{itemize}[topsep=0pt]

\item Can we get matching upper and lower bounds for communication in terms of $m$, $N$, $\log{}d$, and $q$?
\item Currently all of our algorithms work serially. Can we extend the techniques to give parallel speedup?
\item Returning to the general setting \pref{eq:excess_risk}, what
  abstract properties of the hypothesis class $\cH$ are required to guarantee that learning with sublinear communication  is possible?
\end{itemize}

\paragraph{Acknowledgements}
Part of this work was completed while DF was a student at Cornell
University and supported by the Facebook PhD fellowship.

\bibliography{refs}

\appendix
\section{Basic Results}
\label{app:basic}

\subsection{Sparsification}
\label{app:maurey}

In this section we provide approximation guarantees for the Maurey
sparsification operator $\maurey$ defined in \pref{alg:maurey_l1}.
\begin{theorem}
\label{thm:maurey_lp}
Let $p\in\brk*{1,2}$ be fixed. Then for any $w\in\bbR^{d}$, with probability at least $1-\delta$,
\begin{equation}
\label{eq:sparse_high_prob}
\nrm*{Q^{s}(w)-w}_{p} \leq 
4\nrm*{w}_{1} \prn*{\frac{1}{s}}^{1-\frac{1}{p}}
+ \nrm*{w}_{1}\prn*{\frac{8\log(1/\delta)}{s}}^{\frac{1}{2}}
 \\
\leq \nrm*{w}_{1} \prn*{\frac{24\log(1/\delta)}{s}}^{1-\frac{1}{p}}.
\end{equation}
Moreover, the following in-expectation guarantee holds:
\begin{equation}
\En\nrm*{Q^{s}(w)-w}_{p} \leq 
\prn*{\En\nrm*{Q^{s}(w)-w}_{p}^{p}}^{1/p} \leq 
4\nrm*{w}_{1} \prn*{\frac{1}{s}}^{1-\frac{1}{p}}.
\end{equation}
\end{theorem}

\begin{proof}[\pfref{thm:maurey_lp}]
Let $B=\nrm*{w}_{1}$, and let
$Z_{\tau}=\nrm*{w}_{1}\sgn(w_{i_{\tau}})e_{i_{\tau}}-w$, and observe
that $\En\brk*{Z_{\tau}}=0$ and
$Q^{s}(w)-w=\frac{1}{s}\sum_{\tau=1}^{s}Z_{\tau}$. Since
$\nrm*{w}_{p}\leq{}B$, we have $\nrm*{Z_{\tau}}_{p}\leq{}2B$, and so \pref{lem:lp_type} implies that with probability at least $1-\delta$,
\begin{align*}
\nrm*{Q^{s}(w)-w}_{p} &\leq{} \frac{2}{s}\cdot\En_{Z}\prn*{\sum_{t=1}^{s}\nrm*{Z_t}_{p}^{p}}^{1/p}
+ B\sqrt{\frac{8\log(1/\delta)}{s}} \\
&\leq{} \frac{4B}{s^{1-\frac{1}{p}}}
+ B\sqrt{\frac{8\log(1/\delta)}{s}}.
\end{align*}
\end{proof}

\begin{lemma}
\label{lem:lp_type}
Let $p\in\brk*{1,2}$. Let $Z_{1},\ldots,Z_{s}$ be a sequence of independent $\bbR^{d}$-valued random variables with $\nrm*{Z_t}_{p}\leq{}B$ almost surely and $\En\brk*{Z_t}=0$. Then with probability at least $1-\delta$,
\[
\nrm*{\frac{1}{s}\sum_{t=1}^{s}Z_t}_{p}\leq{} \frac{2}{s}\cdot\En_{Z}\prn*{\sum_{t=1}^{s}\nrm*{Z_t}_{p}^{p}}^{1/p}
+ B\sqrt{\frac{2\log(1/\delta)}{s}}
\]
Furthermore, a sharper guarantee holds in expectation:
\[
\En_{Z}\nrm*{\frac{1}{s}\sum_{t=1}^{s}Z_t}_{p}\leq{}
\prn*{\En_{Z}\nrm*{\frac{1}{s}\sum_{t=1}^{s}Z_t}_{p}^{p}}^{1/p}\leq{}
 \frac{2}{s}\cdot\En_{Z}\prn*{\sum_{t=1}^{s}\nrm*{Z_t}_{p}^{p}}^{1/p}.
\]

\end{lemma}
\begin{proof}[\pfref{lem:lp_type}]
To obtain the high-probability statement, the first step is to apply
the standard Mcdiarmid-type high-probability uniform convergence bound for Rademacher complexity (e.g. \cite{shalev2014understanding}), which states that with probability at least $1-\delta$,
\[
\nrm*{\frac{1}{s}\sum_{t=1}^{s}Z_t}_{p}\leq{} 2\En_{Z}\En_{\eps}\nrm*{\frac{1}{s}\sum_{t=1}^{s}\eps_tZ_t}_{p}
+ B\sqrt{\frac{2\log(1/\delta)}{s}},
\]
where $\eps\in\pmo^{n}$ are Rademacher random variables. Conditioning on $Z_1,\ldots,Z_n$, we have
\[
\En_{\eps}\nrm*{\frac{1}{s}\sum_{t=1}^{s}\eps_tZ_t}_{p}
\leq{} \prn*{\En_{\eps}\nrm*{\frac{1}{s}\sum_{t=1}^{s}\eps_tZ_t}_{p}^{p}}^{1/p}.
\]
On the other hand, for the in-expectation results, Jensen's inequality and the standard in-expectation symmetrization argument for Rademacher complexity directly yield
\[
\En_{Z}\nrm*{\frac{1}{s}\sum_{t=1}^{s}Z_t}_{p}
\leq{}\prn*{\En_{Z}\nrm*{\frac{1}{s}\sum_{t=1}^{s}Z_t}_{p}^{p}}^{1/p}
\leq{}2\prn*{\En_{Z}\En_{\eps}\nrm*{\frac{1}{s}\sum_{t=1}^{s}\eps_tZ_t}_{p}^{p}}^{1/p}.
\]
From here the proof proceeds in the same fashion for both cases. Let $Z_{t}[i]$ denote the $i$th coordinate of $Z_t$ and let $z_{i}=(Z_1[i],\ldots,Z_s[i])\in\bbR^{s}$. We have
\begin{align*}
\En_{\eps}\nrm*{\frac{1}{s}\sum_{t=1}^{s}\eps_tZ_t}_{p}^{p}
&= \sum_{i=1}^{d}\En_{\eps}\prn*{\frac{1}{s}\sum_{t=1}^{s}\eps_tZ_t[i]}^{p}
\leq{} \sum_{i=1}^{d}\prn*{\En_{\eps}\prn*{\frac{1}{s}\sum_{t=1}^{s}\eps_tZ_t[i]}^{2}}^{p/2},
\end{align*}
where the inequality follows from Jensen's inequality since $p\leq{}2$.
We now use that cross terms in the square vanish, as well as the standard inequality $\nrm*{x}_2\leq\nrm*{x}_{p}$ for $p\leq{}2$:
\begin{align*}
\sum_{i=1}^{d}\prn*{\En_{\eps}\prn*{\frac{1}{s}\sum_{t=1}^{s}\eps_tZ_t[i]}^{2}}^{p/2}= \sum_{i=1}^{d}\prn*{\frac{1}{s^2}\nrm*{z_i}_{2}^{2}}^{p/2}
= \frac{1}{s^{p}}\sum_{i=1}^{d}\nrm*{z_i}_{2}^{p}
\leq{} \frac{1}{s^{p}}\sum_{i=1}^{d}\nrm*{z_i}_{p}^{p}
= \frac{1}{s^{p}}\sum_{t=1}^{s}\nrm*{Z_t}_{p}^{p}.
\end{align*}

\end{proof}

\begin{proof}[\pfref{lem:maurey_prelim}]
We first prove the result for the smooth case.
Let $x$ and $y$ be fixed. Let $B=\nrm*{w}_{1}$, and let us abbreviate $R\ldef{}R_{\infty}$. Let $Z_{\tau}=\tri*{\nrm*{w}_{1}\sgn(w_{i_{\tau}}) e_{i_{\tau}}-w,x}$, and observe that $\En\brk*{Z_{\tau}}=0$ and $\tri*{Q^{s}(w)-w,x}=\frac{1}{s}\sum_{\tau=1}^{s}Z_{\tau}$. Since we have $\nrm*{w}_{1}\leq{}B$ and $\nrm*{x}_{\infty}\leq{}R$ almost surely, one has $\abs*{Z_{\tau}}\leq{}2BR$ almost surely. We can write
\[
\phi(\tri*{Q^{s}(w),x},y) = \phi\prn*{\tri*{w,x} + \frac{1}{s}\sum_{\tau=1}^{s}Z_{\tau},y}.
\]
Using smoothness, we can write
\[
\phi\prn*{\tri*{w,x} + \frac{1}{s}\sum_{\tau=1}^{s}Z_{\tau},y}
\leq{} 
\phi\prn*{\tri*{w,x} + \frac{1}{s}\sum_{\tau=1}^{s-1}Z_{\tau},y}
+ \phi'\prn*{\tri*{w,x} + \frac{1}{s}\sum_{\tau=1}^{s-1}Z_{\tau},y}\cdot{}\frac{Z_{s}}{s}
+ \frac{\beta}{2s^{2}}(Z_{s})^{2}.
\]
Since $\En\brk*{Z_s\mid{}Z_{1},\ldots,Z_{s-1}}=0$, and since $Z_{s}$ is bounded, taking expectation gives
\[
\En_{Z_s}\brk*{\phi\prn*{\tri*{w,x} + \frac{1}{s}\sum_{\tau=1}^{s}Z_{\tau},y}\mid{}Z_1,\ldots,Z_{s-1}}
\leq{} 
\phi\prn*{\tri*{w,x} + \frac{1}{s}\sum_{\tau=1}^{s-1}Z_{\tau},y}
+ \frac{\beta{}B^{2}}{s^{2}}\nrm*{x}_{\infty}^{2}.
\]
Proceeding backwards in the, fashion, we arrive at the inequality
\[
\En_{Z}\phi\prn*{\tri*{w,x} + \frac{1}{s}\sum_{\tau=1}^{s}Z_{\tau},y}
\leq{} \phi(\tri*{w,x},y) + \frac{\beta{}B^2}{s}\nrm*{x}_{\infty}^2.
\]
The final result follows by taking expectation over $x$ and $y$.

For Lipschitz losses, we use Lipschitzness and Jensen's inequality to write
\[
\En\poprisk(\maurey(w)) - \poprisk(w) \leq{} L\sqrt{\En\En_{x}\tri*{\maurey(w)-w,x}^{2}}.
\]
The result now follows by appealing to the result for the smooth case to bound $\En_{x}\tri*{\maurey(w)-w,x}^{2}$, since we can interpret this as the expectation of new linear model loss $\En_{x,y}\tilde{\phi}(\tri*{w',x},y)\ldef\En_{x}(\tri*{w',x}-\tri*{w,x})^{2}$, where $y=\tri*{w,x}$. This loss is $2$-smooth with respect to the first argument, which leads to the final bound.
\end{proof}

\begin{lemma}
\label{lem:maurey_smooth}
Let $w\in\bbR^{d}$ be fixed and let $F:\bbR^{d}\to\bbR$ have
$\beta_q$-Lipschitz gradient with respect to $\ls_q$, where $q\geq{}2$. Then \pref{alg:maurey_l1} guarantees that
\begin{equation}
\En{}F(Q^{s}(w)) \leq{} F(w) + \frac{\beta_q\nrm*{w}_{1}^2}{s}.
\end{equation}
  
\end{lemma}
\begin{proof}[\pfref{lem:maurey_smooth}]
The assumed gradient Lipschitzness implies that for any $w,w'$
\[
F(w) \leq{} F(w') + \tri*{\grad{}F(w'),w-w'} + \frac{\beta_q}{2}\nrm*{w-w'}_{p}^{2},
\]
where $\frac{1}{p}+\frac{1}{q}=1$.
As in the other Maurey lemmas, we write
$Z_{\tau}=(\nrm*{w}_{1}\sgn(w_{i_{\tau}})e_{i_{\tau}}-w)$, so that
$\En\brk*{Z_{\tau}}=0$ and
$Q^{s}(w)-w=\frac{1}{s}\sum_{\tau=1}^{s}Z_{\tau}$. We can now write
\[
\En{}F(Q^{s}(w)) = \En{}F\prn*{w + \frac{1}{s}\sum_{\tau=s}^{s}Z_{\tau}}
\]
Using smoothness, we have

\begin{align*}
\En_{Z_s}F\prn*{w + \frac{1}{s}\sum_{\tau=s}^{s}Z_{\tau}}
&\leq{}F\prn*{w + \frac{1}{s}\sum_{\tau=s}^{s-1}Z_{\tau}}
+ \En_{Z_s}\tri*{\grad{}F\prn*{w +
  \frac{1}{s}\sum_{\tau=s}^{s-1}Z_{\tau}},\frac{Z_{s}}{s}}
+ \frac{\beta_q}{2s^2}\En_{Z_s}\nrm*{Z_s}_{p}^{2} \\
&\leq{}F\prn*{w + \frac{1}{s}\sum_{\tau=s}^{s-1}Z_{\tau}}
+ \frac{\beta_q}{s^{2}}\nrm*{w}_{1}^{2}.
\end{align*}
Proceeding backwards in the same fashion, we get
\[
\En{}F(\maurey(s)) = \En_{Z_1,\ldots,Z_s}F\prn*{w +
  \frac{1}{s}\sum_{\tau=s}^{s}Z_{\tau}}
\leq{} \frac{\beta_q\nrm*{w}_{1}^{2}}{s}.
\]
\end{proof}

\subsection{Approximation for $\ls_{p}$ Norms}
\label{app:bregman_approx}

In this section we work with the regularizer $\cR(\theta) =
\frac{1}{2}\nrm*{\theta}_{p}^{2}$, where $p\in\brk*{1,2}$, and we let
$q$ be such that $\frac{1}{p}+\frac{1}{q}=1$. The main structural
result we establish is a form of \Holder{} smoothness of $\cR$, which
implies that $\ls_1$ bounded vectors can be sparsified while
preserving Bregman divergences for $\cR$, with the quality degrading as
$p\to1$.

\begin{theorem}
\label{thm:lp_bregman_lipschitz}
Suppose that $a,b,c\in\bbR^{d}$ have $\nrm*{a}_1\vee\nrm*{b}_1\vee\nrm*{c}_1\leq{}B$. Then it holds that
\[
D_{\cR}(c\dmid{}a)
- D_{\cR}(c\dmid{}b) \leq{} 5B\nrm*{a-b}_{p} + 4B^{3-p}\nrm*{a-b}_{\infty}^{p-1}.
\]
\end{theorem}

The remainder of this section is dedicated to proving \pref{thm:lp_bregman_lipschitz}.

We use the following generic fact about norms; all
other results in this section are specific to the $\ls_p$ norm
regularizer. For any norm and any $x,y$ with $\nrm*{x}\vee\nrm*{y}\leq{}B$, we have
\begin{equation}
\label{eq:square_triangle}
\nrm*{x}^{2}-\nrm*{y}^{2} \leq{} \nrm*{x-y}^{2} + 2\nrm*{x-y}\nrm*{y}\leq{}4B\nrm*{x-y}.
\end{equation}

To begin, we need some basic approximation properties. We have the following expression:
\begin{equation}
\label{eq:lp_grad}
\grad{}\cR(\theta) = \nrm*{\theta}_{p}^{2-p}\cdot\prn*{
\abs*{\theta_1}^{p-1}\sgn(\theta_1),\ldots,\abs*{\theta_d}^{p-1}\sgn(\theta_d)
}.
\end{equation}

\begin{proposition}
\label{prop:bregman_dual_norm}
For any vector $\theta$,
\[
\nrm*{\grad{}\reg(\theta)}_{q}=\nrm*{\theta}_{p}.
\]
\end{proposition}
\begin{proof}[\pfref{prop:bregman_dual_norm}]
Expanding the expression in \pref{eq:lp_grad}, we have
\begin{align*}
\nrm*{\grad{}\reg(\theta)}_{q}
=\nrm*{\theta}_{p}^{2-p}\cdot\prn*{\sum_{i=1}^{d}\abs*{\theta_i}^{q(p-1)}}^{1/q}.
\end{align*}
Using that $q=\frac{p}{p-1}$, this simplifies to
\[
\nrm*{\theta}_{p}^{2-p}\cdot\nrm*{\theta}_{p}^{p-1} = \nrm*{\theta}_{p}.
\]

\end{proof}

\begin{lemma}
\label{lem:bregman_norm_bound}
Suppose that $\nrm*{a}_{p}\vee\nrm*{b}_{p}\leq{}B$. Then
\[
\breg(a\dmid{}b)\leq{}3B\nrm*{a-b}_{p}.
\]
\end{lemma}
\begin{proof}[\pfref{lem:bregman_norm_bound}]
We write
\[
\breg(a\dmid{}b) = \reg(a) - \reg(b) -\tri*{\grad\reg(b),a-b}.
\]
Using \pref{eq:square_triangle} and the expression for $\reg$, it follows that
\[
\breg(a\dmid{}b) \leq 2B\nrm*{a-b}_{p} -\tri*{\grad\reg(b),a-b}.
\]
This is further upper bounded by
\[
\breg(a\dmid{}b) \leq 2B\nrm*{a-b}_{p} +\nrm*{\grad\reg(b)}_{q}\nrm*{a-b}_{p}.
\]
The result follows by using that $\nrm*{\grad\reg(b)}_{q}=\nrm*{b}_{p}\leq{}B$, by \pref{prop:bregman_dual_norm}.

\end{proof}

\begin{lemma}
\label{lem:hx_holder}
Let $p\in\brk*{1,2}$ and let $h(x) = \abs*{x}^{p-1}\sgn(x)$. Then $h$ is \Holder{}-continuous: 
\[
\abs*{h(x)-h(y)}\leq{}2\abs*{x-y}^{p-1}\quad\forall{}x,y\in\bbR.
\]
\end{lemma}
\begin{proof}[\pfref{lem:hx_holder}]
Fix any $x,y\in\bbR$ and assume $\abs*{x}\geq\abs*{y}$ without loss of generality. We have two cases. First, when $\sgn(x)=\sgn(y)$ we have
\begin{align*}
\abs*{h(x)-h(y)} = \abs*{\abs*{x}^{p-1}-\abs*{y}^{p-1}}=
 \abs*{x}^{p-1}-\abs*{y}^{p-1}\leq{}  \prn*{\abs*{x}-\abs*{y}}^{p-1}\leq{} \abs*{x-y}^{p-1},
\end{align*}
where we have used that $p-1\in\brk*{0,1}$ and subadditivity of $x\mapsto{}x^{p-1}$ over $\bbR_{+}$, as well as triangle inequality.
On the other hand if $\sgn(x)\neq\sgn(y)$, we have 
\[
\abs*{h(x)-h(y)} = \abs*{\abs*{x}^{p-1}+\abs*{y}^{p-1}}
= \abs*{x}^{p-1}+\abs*{y}^{p-1} \leq{} 2^{2-p}\abs*{\abs*{x}+\abs*{y}}^{p-1}.
\]
Now, using that $\sgn(x)\neq\sgn(y)$, we have
\[
\abs*{\abs*{x}+\abs*{y}}^{p-1}=\abs*{\abs*{x}\cdot\sgn(x)+\abs*{y}\cdot\sgn(x)}^{p-1}
=\abs*{\abs*{x}\cdot\sgn(x)-\abs*{y}\cdot\sgn(y)}^{p-1}
=\abs*{x-y}^{p-1}.
\]
Putting everything together, this establishes that
\[
\abs*{h(x)-h(y)}\leq{}2^{2-p}\abs*{x-y}^{p-1}\leq{}2\abs*{x-y}^{p-1}.
\]

\end{proof}

\begin{lemma}
\label{lem:bregman_gradient_holder}
Suppose that $\nrm*{a}_{p}\vee\nrm*{b}_{p}\leq
{}B$. 
Then it holds that
\begin{equation}
\nrm*{\grad\reg(a)-\grad\reg(b)}_{\infty}
\leq{}2B^{2-p}\nrm*{a-b}_{\infty}^{p-1} + \nrm*{a-b}_{p},
\end{equation}
and
\begin{equation}
\nrm*{\grad\reg(a)-\grad\reg(b)}_{q}
\leq{}2B^{2-p}\nrm*{a-b}_{p}^{p-1} + \nrm*{a-b}_{p}.
\end{equation}
\end{lemma}
\begin{proof}[\pfref{lem:bregman_gradient_holder}]
Let $h(x) = \abs*{x}^{p-1}\sgn(x)$, so that
\[
\grad{}\cR(\theta) = \nrm*{\theta}_{p}^{2-p}\cdot\prn*{
h(\theta_1),\ldots,h(\theta_d)
}.
\]

Fix vectors $a,b\in\bbR^{d}$. Assume without loss of generality that $\nrm*{a}_{p}\geq{}\nrm*{b}_{p}>0$; if $\nrm*{b}_{p}=0$ the result follows immediately from \pref{prop:bregman_dual_norm}. We work with the following normalized vectors: $\bar{a}\ldef{}a/\nrm*{b}_{p}$ and $\bar{b}\ldef{}b/\nrm*{b}_{p}$. Our assumptions on the norms imply $\nrm*{\bar{a}}_{p}\geq{}\nrm*{\bar{b}}_{p}=1$.

Fix a coordinate $i\in\brk*{d}$. We establish the following chain of elementary inequalities:
\begin{align*}
\abs*{
\grad\reg(\bar{a})_i
-\grad\reg(\bar{b})_i
} &=
\abs*{
\nrm*{\bar{a}}_{p}^{2-p}h(\bar{a}_i)
-\nrm*{\bar{b}}_{p}^{2-p}h(\bar{b}_i)
} \\
&=
\abs*{
\nrm*{\bar{a}}_{p}^{2-p}h(\bar{a}_i)
- \nrm*{\bar{a}}_{p}^{2-p}h(\bar{b}_i)
+ \nrm*{\bar{a}}_{p}^{2-p}h(\bar{b}_i)
-\nrm*{\bar{b}}_{p}^{2-p}h(\bar{b}_i)
} 
\intertext{Using the triangle inequality:}
&\leq
\nrm*{\bar{a}}_{p}^{2-p}\cdot
\abs*{
h(\bar{a}_i)
- h(\bar{b}_i)
}
+\abs*{\bar{b}_i}^{p-1}\cdot\abs*{
\nrm*{\bar{a}}_{p}^{2-p}
-\nrm*{\bar{b}}_{p}^{2-p}
}
\intertext{Using the \Holder{}-continuity of $h$ established in \pref{lem:hx_holder}:}
&\leq
2\nrm*{\bar{a}}_{p}^{2-p}\cdot
\abs*{
\bar{a}_i-\bar{b}_i
}^{p-1}
+\abs*{\bar{b}_i}^{p-1}\cdot\abs*{
\nrm*{\bar{a}}_{p}^{2-p}
-\nrm*{\bar{b}}_{p}^{2-p}
}
\intertext{Using that $\nrm*{\bar{a}}_{p}\geq{}\nrm*{\bar{b}}_{p}=1$:}
&\leq
2\nrm*{\bar{a}}_{p}^{2-p}\cdot
\abs*{
\bar{a}_i-\bar{b}_i
}^{p-1}
+\abs*{\bar{b}_i}^{p-1}\cdot\prn*{
\nrm*{\bar{a}}_{p}^{2-p}
-1
}.
\intertext{Finally, since $\nrm*{\bar{a}}_{p}\geq{}1$ and $2-p\leq{}1$, we can drop the exponent:}
&\leq
2\nrm*{\bar{a}}_{p}^{2-p}\cdot
\abs*{
\bar{a}_i-\bar{b}_i
}^{p-1}
+\abs*{\bar{b}_i}^{p-1}\cdot\prn*{
\nrm*{\bar{a}}_{p}
-1
}.
\end{align*}
To finish the proof, we rescale both sides of the inequality by $\nrm*{b}_{p}$. Observe that $\grad\cR(\theta)$ is homogeneous in the following sense: For any $r\geq0$,
\[
\grad\reg(r\theta) = r\cdot\grad\reg(\theta).
\]
Along with this observation, the inequality we just established implies
\begin{align*}
\abs*{
\grad\reg(a)_i
-\grad\reg(b)_i
}
 & \leq{} 
 2\nrm*{b}_{p}\nrm*{\bar{a}}_{p}^{2-p}\cdot
\abs*{
\bar{a}_i-\bar{b}_i
}^{p-1}
+
\abs*{\bar{b}_i}^{p-1}\cdot\prn*{\nrm*{a}_{p}
-\nrm*{b}_{p}} \\
 & \leq{} 
 2\nrm*{b}_{p}\nrm*{\bar{a}}_{p}^{2-p}\cdot
\abs*{
\bar{a}_i-\bar{b}_i
}^{p-1}
+
\abs*{\bar{b}_i}^{p-1}\cdot\nrm*{a-b}_{p}\\
 & =
 2\prn*{\nrm*{\bar{a}}_{p}\nrm*{b}_{p}}^{2-p}\cdot
\abs*{
\bar{a}_i\nrm*{b}_{p}-\bar{b}_i\nrm*{b}_{p}
}^{p-1}
+
\abs*{\bar{b}_i}^{p-1}\cdot\nrm*{a-b}_{p}\\
 & =
 2\nrm*{a}_{p}^{2-p}\cdot
\abs*{
a_{i}-b_{i}
}^{p-1}
+
\abs*{\bar{b}_i}^{p-1}\cdot\nrm*{a-b}_{p}.
\end{align*}
For the $\ls_{\infty}$ bound, the result follows immediately by using that $\abs*{\bar{b}_i}\leq{}\nrm*{\bar{b}}_{p}\leq{}1$. For the $\ls_{q}$ bound, we use that for any vector $z$, $\nrm*{(z_i^{p-1})_{i\leq{}d}}_{q}=\nrm*{z}_{p}^{p-1}$, and that $\nrm*{\bar{b}}_{p}\leq{}1$.

\end{proof}

\begin{proof}[\pfref{thm:lp_bregman_lipschitz}]
Throughout this proof we use that $\nrm*{x}_{p}\leq{}\nrm*{x}_1$ for
all $p\geq{}1$. To start, expanding the definition of the Bregman
divergence we have
\begin{align*}
D_{\cR}(c\dmid{}a)
- D_{\cR}(c\dmid{}b)
&= D_{\cR}(b\dmid{}a) + \tri*{\grad\cR(a)-\grad\cR(b),b-c}.\\
\intertext{Using \pref{lem:bregman_norm_bound}, this is at most}
&= 3B\nrm*{a-b}_{p} + \tri*{\grad\cR(a)-\grad\cR(b),b-c}.\\
\intertext{Now, applying \Holder{}'s inequality, this is upper bounded by}
&\leq 3B\nrm*{a-b}_{p} + \nrm*{\grad\cR(a)-\grad\cR(b)}_{\infty}\nrm*{b-c}_{1}\\
&\leq{} 3B\nrm*{a-b}_{p} + 2B\nrm*{\grad\cR(a)-\grad\cR(b)}_{\infty}.
\end{align*}
To conclude, we plug in the bound from \pref{lem:bregman_gradient_holder}.
\end{proof}


\section{Proofs from \pref{sec:overview}}
\subsection{Proofs from \pref{sec:l2l2}}
\label{app:JL}

\begin{proof}[\pfref{thm:jl}]
Let $A\in\bbR^{k\times{}d}$ be the derandomized JL matrix constructed according to \citet{deJL}, Theorem 2. Let $x'_t=Ax_t$ denote the projected feature vector and $w^\star = \argmin_{w : \|w\|_2 \le 1} L_{\cD}(w)$. 

We first bound the regret of gradient descent in the projected space in terms of certain quantities that depend on $A$, then show how the JL matrix construction guarantees that these quantities are appropriately bounded.

Since $\phi$ is $L$-Lipschitz, we have the preliminary error estimate
\begin{align*}
\phi(\tri{Ax, A w^\star}, y) -  \phi(\tri{x,  w^\star}, y) 
&\le L \ \left|\left< A x, A w^\star\right> - \left<x,  w^\star\right>\right|,
\end{align*}
and so 
\begin{align} \label{eq:lip}
L_{\cD}(A^\top A w^\star) - L_{\cD}(w^\star)  \le L\cdot\mathbb{E}_{x}\left|\left< A x, A w^\star\right> - \left<x,  w^\star\right>\right|.
\end{align}
Now recall that the $m$ machines are simply running online gradient descent in serial over the $k$-dimensional projected space, and the update has the form $u_t \gets u_{t-1} - \nabla \phi(\tri{u_t,x'_t}, y_t)$, where $\eta$ is the learning rate parameter. The standard online gradient descent regret guarantee \citep{hazan2016introduction} implies that for any vector $u \in \mathbb{R}^k$:
$$
\frac{1}{N} \sum_{t=1}^N \phi(\tri{u_t,x'_t},y_t) - \frac{1}{N} \sum_{t=1}^N \phi(\tri{u,x'_t},y_t) \le \frac{1}{2 \eta N}\|u\|_2^2 + \frac{\eta}{2 N} \sum_{t=1}^N \|x'_t\|_2^2.
$$
Equivalently, we have
$$
\frac{1}{N} \sum_{t=1}^N \phi(\tri{A^{\trn}u_t,x_t}, y_t) - \frac{1}{N} \sum_{t=1}^N \phi(\tri{A^\top u,x_t}, y_t) \le \frac{1}{2 \eta N}\|u\|_2^2 + \frac{\eta}{2 N} \sum_{t=1}^N \|A x_t\|_2^2
$$
Since the pairs $(x_t,y_t)$ are drawn i.i.d.,  the standard online-to-batch conversion lemma for online convex optimization \citep{PLG} yields the following guarantee for any vector $u$:
\begin{align*}
\frac{1}{N} \sum_{t=1}^N \mathbb{E}_S\left[L_{\cD}(A^\top u_t)\right] - L_{\cD}(A^\top u) &\le \frac{1}{2 \eta N}\|u\|_2^2 + \frac{\eta}{2 N} \sum_{t=1}^N \mathbb{E}_S \|A x_t\|_2^2\\
& = \frac{1}{2 \eta N}\|u\|_2^2 + \frac{\eta{}L^2}{2}  \mathbb{E}_x \|A x\|_2^2.
\end{align*}
Applying Jensen's inequality to the left-hand side and choosing $u = u^\star \ldef A w^\star$, we conclude that
\begin{align*}
 \mathbb{E}_S\left[L_{\cD}\left(\frac{1}{N} \sum_{t=1}^N A^\top u_t\right)\right] - L_{\cD}(A^\top u^\star) &\le \frac{1}{2 \eta N}\|A w^\star\|_2^2 + \frac{\eta{}L^2}{2}  \mathbb{E}_x \|A x\|_2^2,
\end{align*}
or in other words,
$$
 \mathbb{E}_S\left[L_{\cD}\left(\hat{w}\right)\right] - L_{\cD}(A^\top A w^\star) \le \frac{1}{2 \eta N}\|A w^\star\|_2^2 + \frac{\eta{}L^2}{2}  \mathbb{E}_x \|A x\|_2^2.
$$
We now relate this bound to the risk relative to the benchmark $\poprisk(w^{\star})$. Using \pref{eq:lip} we have
\begin{align*}
\mathbb{E}_S\left[L_{\cD}\left(\hat{w}\right)\right] - L_{\cD}(w^\star) \le  \frac{1}{2 \eta N}\|A w^\star\|_2^2 + \frac{\eta{}L^2}{2}  \mathbb{E}_x \|A x\|_2^2 + L \mathbb{E}_{x}\left|\left< A x, A w^\star\right> - \left<x,  w^\star\right>\right|.
\end{align*}
Taking expectation with respect to the draw $A$, we get that
\begin{align}\label{eq:inter}
\mathbb{E}_S \mathbb{E}_A \left[L_{\cD}\left(\hat{w}\right)\right] - L_{\cD}(w^\star) \le \mathbb{E}_x\left[ \mathbb{E}_A\left[\frac{1}{2 \eta N}\|A w^\star\|_2^2 + \frac{\eta{}L^2}{2}  \|A x\|_2^2 + L \left|\left< A x, A w^\star\right> - \left<x,  w^\star\right>\right| \right] \right].
\end{align}
It remains to bound the right-hand side of this expression. To begin, we condition on the vector $x$ with respect to which the outer expectation in \pref{eq:inter} is taken.
The derandomized JL transform guarantees (\citet{deJL}, Theorem 2) that for any $\delta >0$ and any fixed vectors $x, w^\star$, if we pick $k = O\left(\log(1/\delta)/\veps^2\right)$, then we are guaranteed that with probability at least $1 - \delta$,
$$
\|A x\|_2 \le (1 + \veps) \|x\|_2  ,~~~ \|A w^\star\|_2 \le (1 + \veps) \|w^\star\|_2  ~~~~\text{and}~~~~  \left|\left< A x, A w^\star\right> - \left<x,  w^\star\right>\right| \le \frac{\veps}{4}\nrm*{x}_{2}\nrm*{w^{\star}}_{2}. 
$$
We conclude that by picking $\veps = O\prn*{1/\sqrt{N}}$, 
with probability $1 - \delta$,
$$
\|A x\|_2 \le O(R_2) ,~~~ \|A w^\star\|_2 \le O(B_2),  ~~~~\text{and}~~~~  \left|\left< A x, A w^\star\right> - \left<x,  w^\star\right>\right| \le O\prn*{\frac{B_2R_2}{\sqrt{N}}}.
$$
To convert this into an in-expectation guarantee, note that since entries in $A$ belong to $\crl*{-1,0,+1}$, the quantities $\|A x\|_2$, $\|A w^\star\|_2$, and $\left< A x, A w^\star\right>$ all have magnitude $O(\poly(d))$ with probability $1$ (up to scale factors $B_2$ and $R_2$). Hence, 
\begin{align*}
&\mathbb{E}_A\left[\frac{1}{2 \eta N}\|A w^\star\|_2^2 + \frac{\eta{}L^2}{2}  \|A x\|_2^2 + L \left|\left< A x, A w^\star\right> - \left<x,  w^\star\right>\right| \right]  \\
&\le  (1 - \delta)\cdot O\left( \frac{B_2^2}{2 \eta N} + \frac{\eta{}L^2R_2^2}{2}   + \frac{LB_2R_2}{\sqrt{N}}  \right)  + \delta\cdot{} O\left(\poly{}(d)\cdot\prn*{\frac{B_2^2}{2 \eta N} + \frac{\eta{}L^2R_2^2}{2}   + LB_2R_2}  \right).
\end{align*}
Picking $\delta = 1/\sqrt{\poly{}(d)N}$ and using the step size $\eta = \sqrt{\frac{B_2^2}{L^{2}R_2^{2}N}}$, we get the desired bound:
\begin{align*}
\mathbb{E}_A\left[\frac{1}{2 \eta N}\|A w^\star\|_2^2 + \frac{\eta{}L^2}{2}  \|A x\|_2^2 + L \left|\left< A x, A w^\star\right> - \left<x,  w^\star\right>\right| \right]  & \le  O(LB_2R_2/\sqrt{N}).
\end{align*}
Since this in-expectation guarantee holds for any fixed $x$, it also holds in expectation over $x$:
\begin{align*}
\En_{x}\mathbb{E}_A\left[\frac{1}{2 \eta N}\|A w^\star\|_2^2 + \frac{\eta{}L^2}{2}  \|A x\|_2^2 + L \left|\left< A x, A w^\star\right> - \left<x,  w^\star\right>\right| \right]  & \le   O(L/\sqrt{N}).
\end{align*}

Using this inequality to bound the right-hand side in \pref{eq:inter} yields the claimed excess risk bound. Recall that we have $k = O\left(\log(1/\delta)/\veps^2\right) = O\left(N \log(N d)\right)$, and so the communication cost to send a single iterate (taking into account numerical precision) is upper bounded by $O(N \log(N d)\cdot\log(LB_2R_2N))$. 
\end{proof}


\subsection{Proofs from \pref{sec:lower_bounds}}
\label{app:lower_bounds}

Our lower bounds are based on reduction to the so-called ``hide-and-seek'' problem introduced by \citet{shamir2014fundamental}.
\begin{definition}[Hide-and-seek problem]
Let $\crl*{\bbP_{j}}_{j=1}^{d}$ be a set of product distributions over $\pmo^{d}$ defined via $\En_{\bbP_{j}}\brk*{z_{i}}=2\rho\ind\crl*{j=i}$. Given $N$ i.i.d. instances from $\bbP_{j^{\star}}$, where $j^{\star}$ is unknown, detect $j^{\star}$.
\end{definition}

\begin{theorem}[\citet{shamir2014fundamental}]
\label{thm:hide_and_seek}
Let $W\in\brk*{d}$ be the output of a $(b,1,N)$ protocol for the hide-and-seek problem. Then there exists some $j^{\star}\in\brk*{d}$ such that
\[
\Pr_{j^{\star}}\prn*{W=j^{\star}}\leq{} \frac{3}{d} + \sqrt{\frac{Nb\rho^{2}}{d}}.
\]

\end{theorem}

\begin{proof}[\pfref{thm:lower_bound_slow}]
Recall that $\cW_1=\crl*{w\in\bbR^{d}\mid\nrm*{w}_1\leq{}1}$. We create a family of $d$ statistical learning instances as follows. Let the hide-and seek parameter $\rho\in\brk*{0,1/2}$ be fixed. Let $\cD_{j}$ have features drawn from the be the $j$th hide-and-seek distribution $\bbP_{j}$ and have $y=1$, and set $\phi(\tri*{w,x},y)=-\tri*{w,x}y$, so that $L_{\cD_j}(w)=-2\rho{}w_j$. Then we have $\min_{w\in\cW_{1}}L_{\cD_j}(w)=-2\rho$. Consequently, for any predictor weight vector $w$ we have
\[
L_{\cD_j}(w)-L_{\cD_j}(w^{\star}) = 2\rho(1-w_j).
\]
If $L_{\cD_j}(\wh{w})-L_{\cD_j}(w^{\star}) < \rho$, this implies (by rearranging) that $\wh{w}_j>\frac{1}{2}$. Since $\wh{w}\in\cW_1$ and thus $\sum_{i=1}^{d}\abs*{\wh{w}_j}\leq{}1$, this implies $j=\argmax_{i}\wh{w}_i$. Thus, if we define $W=\argmax_{i}\wh{w}$ as our decision for the hide-and-seek problem, we have
\[
\Pr_{j}\prn*{L_{\cD_j}(\wh{w}) - L_{\cD_j}(w^{\star})<\rho} \leq{} \Pr_{j}\prn*{W=j}.
\]
Appealing to \pref{thm:hide_and_seek}, this means that for every algorithm $\wh{w}$ there exists an index $j$ for which
\[
\Pr_{j}\prn*{L_{\cD_j}(\wh{w}) - L_{\cD_j}(w^{\star})<\rho} \leq{} 
\frac{3}{d} + \sqrt{\frac{Nb\rho^{2}}{d}}.
\]
To conclude the result we choose $\rho=\frac{1}{16}\sqrt{\frac{d}{bN}}\wedge{}\frac{1}{2}$.

\end{proof}

\begin{proof}[\pfref{prop:lb_multi_machine}]
This result is an immediate consequence of the reductions to the
hide-and-seek problem established
in \pref{thm:lower_bound_slow}. All that
changes is which lower bound for the hide-and-seek problem we
invoke. We set $\rho\propto\frac{d}{b{}N}$ in the
construction in \pref{thm:lower_bound_slow}, then appeal to Theorem 3 in \citet{shamir2014fundamental}.
\end{proof}

\begin{proof}[\pfref{prop:lower_bound_fast}]
We create a family of $d$ statistical learning instances as follows. Let the hide-and seek parameter $\rho\in\brk*{0,1/2}$ be fixed. Let $\bbP_{j}$ be the $j$th hide-and-seek distribution. We create distribution $\cD_{j}$ via: 1) Draw $x\sim{}\bbP_j$ 2) set $y=1$. Observe that $\En\brk*{x_ix_k}=0$ for all $i\neq{}k$ and $\En\brk*{x_i^{2}}=1$, so $\Sigma=I$. Consequently, we have
\begin{align*}
L_{\cD_j}(w) = \En_{x\sim{}\bbP_j}\prn*{\tri*{w,x}-y}^{2} = w^{\trn}\Sigma{}w - 4\rho{}w_j + 1 = \nrm*{w}_{2}^{2} - 4\rho{}w_j + 1.
\end{align*}
Let $w^{\star}=\argmin_{w\in\nrm*{w}_1\leq{}1}L_{\cD_j}(w)$. It is clear from the expression above $w^{\star}_i=0$ for all $i\neq{}j$. For coordinate $j$ we have $w^{\star}_j=\argmin_{-1\leq\alpha\leq{}1}\crl*{\alpha^{2} - 4\rho\alpha}$. Whenever $\rho\leq{}1/2$ the solution is $2\rho$, so we can write $w^{\star}=2\rho{}e_j$, which is clearly $1$-sparse.

We can now write the excess risk for a predictor $w$ as 
\[
L_{\cD_j}(w) - L_{\cD_j}(w^{\star}) = \nrm*{w}_{2}^{2} - 4\rho{}w_j + 4\rho^{2} =  \sum_{i\neq{}j}w_i^{2} + (w_j-2\rho)^{2}.
\]
Now suppose that the excess risk for $w$ is at most $\rho^{2}$. Dropping the sum term in the excess risk, this implies
\[
(w_j-2\rho)^{2}< \rho^{2}.
\]
It follows that $w_{j}\in(\rho,3\rho)$. On the other hand, we also have
\[
\sum_{i\neq{}j}w_i^{2} < \rho^{2},
\]
and so any $i\neq{}j$ must have $\abs*{w_i}<\rho$. Together, these facts imply that if the excess risk for $w$ is less than $\rho^{2}$, then $j=\argmax_{i}w_i$.

Thus, for any algorithm output $\wh{w}$, if we define $W=\argmax_{i}\wh{w}_i$ as our decision for the hide-and-seek problem, we have
\[
\Pr_{j}\prn*{L_{\cD_j}(\wh{w}) - L_{\cD_j}(w^{\star})<\rho^{2}} \leq{} \Pr_{j}\prn*{W=j}.
\]
The result follows by appealing to Theorem 2 and Theorem 3 in \cite{shamir2014fundamental}.
\end{proof}

\subsection{Discussion: Support Recovery} Our lower bound for the sparse regression setting \pref{eq:sparse_risk} does not rule out the possibility of sublinear-communication distributed algorithms for well-specified models. Here we sketch a strategy that works for this setting if we significantly strengthen the statistical assumptions.

Suppose that we work with the square loss and labels are realized as
$y=\tri*{w^{\star},x}+\veps$, where $\veps$ is conditionally mean-zero
and $w^{\star}$ is $k$-sparse. Suppose in addition that the population
covariance $\Sigma$ has the restricted eigenvalue property, and that
$w^{\star}$ satisfies the so-called ``$\beta$-min'' assumption: All
non-zero coordinates of $w^{\star}$ have magnitude bounded below.

In this case, if $N/m=\Omega(k\log{}d)$ and the smallest non-zero
coefficients of $w^{\star}$ are at least $\wt{\Omega}(\sqrt{m/N})$ the
following strategy works: For each machine, run Lasso on the first
half of the examples to exactly recover the support of $w^{\star}$
(e.g. \citet{loh2017support}). On the second half of examples,
restrict to the recovered support and use the
strategy from \citet{zhang2012communication}: run ridge regression on
each machine locally with an appropriate choice of regularization parameter, then send all ridge regression estimators to a central server that averages them and returns this as the final estimator.

This strategy has $O(mk)$ communication by definition, but the
assumptions on sparsity and $\beta$-min depend on the number of
machines. How far can these assumptions be weakened?


\section{Proofs from \pref{sec:upper_bounds}}
\label{app:upper_bounds}

Throughout this section of the appendix we adopt the shorthand
$B\ldef{}B_1$ and $R\ldef{}R_q$. Recall that $\frac{1}{p}+\frac{1}{q}=1$.

To simplify expressions throughout the proofs in this section we use the convention
$\wh{w}^{0}\ldef\bar{w}$ and $\wt{w}^{i}\ldef{}w_{n+1}^{i}$.

We begin the section by stating a few preliminary results used to analyze the
performance of \pref{alg:mirror_slow} and
\pref{alg:mirror_fast}. We then proceed to prove the main theorems.

For the results on fast rates we need the following intermediate fact,
which states that centering the regularizer $\cR$ at $\bar{w}$ does
not change the strong convexity from \pref{prop:lp_strong_convexity} or
smoothness properties established in \pref{app:bregman_approx}.
\begin{proposition}
\label{prop:lp_centered}
Let $\reg(w) = \frac{1}{2}\nrm*{w-\bar{w}}_{p}^{2}$, where
$\nrm*{w}_1\leq{}B$.  Then $\breg(a\dmid{}b)\geq\frac{p-1}{2}\nrm*{a-b}_{p}^{2}$ and if $\nrm*{a}_1\vee\nrm*{b}_1\vee\nrm*{c}_1\leq{}B$ it holds that
\[
D_{\cR}(c\dmid{}a)
- D_{\cR}(c\dmid{}b) \leq{} 10B\nrm*{a-b}_{p} + 16B^{3-p}\nrm*{a-b}_{\infty}^{p-1}.
\]
\end{proposition}
\begin{proof}[\pfref{prop:lp_centered}]
Let $\reg_{0}(w)=\frac{1}{2}\nrm*{w}_{p}^{2}$. The
result follows from \pref{prop:lp_strong_convexity} and
\pref{thm:lp_bregman_lipschitz} by simply observing that
$\grad\reg(w) = \grad{}\reg_{0}(w-\bar{w})$ so that
$D_{\reg}(w\dmid{}w') =
D_{\reg_{0}}(w-\bar{w}\dmid{}w'-\bar{w})$. To invoke
\pref{thm:lp_bregman_lipschitz} we use that $\nrm*{a-\bar{w}}_{1}\leq{}2B$, and likewise
for $b$ and $c$.
\end{proof}

\begin{lemma}
\label{lem:mirror_single}
\pref{alg:mirror_slow} guarantees that for any adaptively selected
sequence $\grad_{t}^{i}$ and all $w^{\star}\in\cW$, any individual
machine $i\in\brk*{m}$ deterministically satisfies the following guarantee:
\[
\sum_{t=1}^{n}\tri*{\grad_{t}^{i},w_{t}^{i}-w^{\star}}
\leq{} \frac{\eta{}C_{q}}{2}\sum_{t=1}^{n}\nrm*{\grad_{t}^{i}}_{q}^{2} + \frac{1}{\eta}\prn*{
\breg(w^{\star}\dmid{}w_{1}^{i}) - \breg(w^{\star}\dmid{}w_{n+1}^{i})
}
\]
\end{lemma}
\begin{proof}[\pfref{lem:mirror_single}]
This is a standard argument. Let $w^{\star}\in\cW$ be fixed. The standard Bregman divergence
inequality for mirror descent \citep{ben2001lectures} implies that for
every time $t$, we have
\[
\tri*{\grad_{t}^{i},w_{t}^{i}-w^{\star}}
\leq{} \tri*{\grad_{t}^{i},w_{t}^{i}-\theta_{t+1}^{i}}
+ \frac{1}{\eta{}}
\prn*{
\breg(w^{\star}\dmid{}w_{t}^{i}) - \breg(w^{\star}\dmid{}w_{t+1}^{i})
- \breg(w_{t}^{i}\dmid{}\theta_{t+1}^{i})
}.
\]
Using \pref{prop:lp_centered}, we have an upper bound of
\begin{align*}
&\tri*{\grad_{t}^{i},w_{t}^{i}-\theta_{t+1}^{i}}
+ \frac{1}{\eta{}}
\prn*{
\breg(w^{\star}\dmid{}w_{t}^{i}) - \breg(w^{\star}\dmid{}w_{t+1}^{i})
- \frac{p-1}{2}\nrm*{w_{t}^{i}-\theta_{t+1}^{i}}_{p}^{2}
}.
\intertext{Using \Holder's inequality and AM-GM:}
&\leq{} \frac{\eta{}}{2(p-1)}\nrm*{\grad_{t}^{i}}_{q}^{2} + \frac{p-1}{2\eta{}}\nrm*{w_{t}^{i}-\theta_{t+1}^{i}}_{p}^{2}
+ \frac{1}{\eta{}}
\prn*{
\breg(w^{\star}\dmid{}w_{t}^{i}) - \breg(w^{\star}\dmid{}w_{t+1}^{i})
- \frac{p-1}{2}\nrm*{w_{t}^{i}-\theta_{t+1}^{i}}_{p}^{2}
} \\
&= \frac{\eta{}}{2(p-1)}\nrm*{\grad_{t}^{i}}_{q}^{2} 
+ \frac{1}{\eta{}}
\prn*{
\breg(w^{\star}\dmid{}w_{t}^{i}) - \breg(w^{\star}\dmid{}w_{t+1}^{i})
}.
\end{align*}
The result follows by summing across time and observing that the Bregman divergences telescope.
\end{proof}

\begin{proof}[\pfref{thm:stochastic_mirror_slow}]
To begin, the guarantee from \pref{lem:mirror_single} implies that for
any fixed machine $i$, deterministically,
\begin{align*}
\sum_{t=1}^{n}\tri*{\grad_{t}^{i},w_{t}^{i}-w^{\star}}
&\leq{} \frac{\eta{}C_{q}}{2}\sum_{t=1}^{n}\nrm*{\grad_{t}^{i}}_{q}^{2} + \frac{1}{\eta}\prn*{
\breg(w^{\star}\dmid{}w_{1}^{i}) - \breg(w^{\star}\dmid{}w_{n+1}^{i})
}.
\end{align*}
We now use the usual reduction from regret to stochastic optimization: since $w_{t}^{i}$ does not depend on $\grad_{t}^{i}$, we can take expectation over $\grad_{t}^{i}$ to get
\[
\En\brk*{\sum_{t=1}^{n}\tri*{\grad{}\poprisk(w_{t}^{i}),w_{t}^{i}-w^{\star}}}\leq{} 
\frac{\eta{}C_{q}}{2}\sum_{t=1}^{n}\En\nrm*{\grad_{t}^{i}}_{q}^{2} + \frac{1}{\eta}\En\brk*{
\breg(w^{\star}\dmid{}w_{1}^{i}) - \breg(w^{\star}\dmid{}w_{n+1}^{i})}
\]
and furthermore, $\poprisk$ is convex, this implies
\[
\En\brk*{\sum_{t=1}^{n}\poprisk(w_{t}^{i})-\poprisk(w^{\star})
}\leq{} 
\frac{\eta{}C_{q}}{2}\sum_{t=1}^{n}\En\nrm*{\grad_{t}^{i}}_{q}^{2} + \frac{1}{\eta}\En\brk*{
\breg(w^{\star}\dmid{}w_{1}^{i}) - \breg(w^{\star}\dmid{}w_{n+1}^{i})}.
\]
While the regret guarantee implies that this holds for each machine
$i$ conditioned on the history up until the machine begins working, it
suffices for our purposes to interpret the expectation above as with
respect to all randomness in the algorithm's execution except for the
randomness in sparsification for the final iterate $\wh{w}$.

We now sum this guarantee across all machines, which gives
\begin{align*}
\En\brk*{\sum_{i=1}^{m}\sum_{t=1}^{n}\poprisk(w_{t}^{i})-\poprisk(w^{\star})
}
&\leq{} \frac{\eta{}C_{q}}{2}\sum_{i=1}^{m}\sum_{t=1}^{n}\En\nrm*{\grad_{t}^{i}}_{q}^{2} + \frac{1}{\eta}\sum_{i=1}^{m}\En\brk*{
\breg(w^{\star}\dmid{}w_{1}^{i}) - \breg(w^{\star}\dmid{}w_{n+1}^{i})
}.
\intertext{Rewriting in terms of $\wt{w}^{i}$ and its sparsified version $\wh{w}^{i}$ and using that $w_{1}^{1}=\bar{w}$, this is upper bounded by}
&\leq{} \frac{\eta{}C_{q}}{2}\sum_{i=1}^{m}\sum_{t=1}^{n}\En\nrm*{\grad_{t}^{i}}_{q}^{2} + \frac{\breg(w^{\star}\dmid{}\bar{w})}{\eta} + \frac{1}{\eta}\sum_{i=1}^{m-1}\En\brk*{
\breg(w^{\star}\dmid{}\wh{w}^{i}) - \breg(w^{\star}\dmid{}\wt{w}^{i})}.
\end{align*}
We now bound the approximation error in the final term. Using \pref{prop:lp_centered}, we get
\[
\sum_{i=1}^{m-1}\En\brk*{
\breg(w^{\star}\dmid{}\wh{w}^{i}) - \breg(w^{\star}\dmid{}\wt{w}^{i})
}\leq{} O\prn*{\sum_{i=1}^{m-1}
B\En\nrm*{\wh{w}^{i}-\wt{w}^{i}}_{p} + B^{3-p}\En\nrm*{\wh{w}^{i}-\wt{w}^{i}}_{\infty}^{p-1}}.
\]
\pref{thm:maurey_lp} implies that
$\En\nrm*{\wh{w}^{i}-\wt{w}^{i}}_{p} 
\leq O\prn*{
B\prn*{\frac{1}{s}}^{1-\frac{1}{p}}
}$ and $\En\nrm*{\wh{w}^{i}-\wt{w}^{i}}_{\infty}^{p-1}
\leq O\prn*{
B^{p-1}\prn*{\frac{1}{s}}^{\frac{p-1}{2}}
}$.\footnote{The second bound follows by appealing to the $\ls_2$ case
  in \pref{thm:maurey_lp} and using that
  $\nrm*{x}_{\infty}\leq{}\nrm*{x}_{2}$.}
In particular, we get
\begin{align*}
\sum_{i=1}^{m-1}\En\brk*{
\breg(w^{\star}\dmid{}\wh{w}^{i}) - \breg(w^{\star}\dmid{}\wt{w}^{i})
}\leq{} O\prn*{\sum_{i=1}^{m-1}
B^{2}\prn*{\frac{1}{s}}^{1-\frac{1}{p}} + B^{3-p}\cdot{}B^{p-1}\prn*{\frac{1}{s}}^{\frac{1}{2}}
} 
= O\prn*{B^{2}\sum_{i=1}^{m-1}
\prn*{\frac{1}{s}}^{1-\frac{1}{p}} + \prn*{\frac{1}{s}}^{\frac{p-1}{2}}
}.
\end{align*}
Since $p\leq{}2$, the second summand dominates, leading to a final bound of
$O\prn*{B^{2}m\prn*{\frac{1}{s}}^{\frac{p-1}{2}}
}$. To summarize, our developments so far (after normalizing by $N$) imply
\begin{align*} 
\En\brk*{\frac{1}{mn}\sum_{i=1}^{m}\sum_{t=1}^{n}\poprisk(w_{t}^{i})-\poprisk(w^{\star})
}
&\leq{}
  \frac{\eta{}C_{q}}{2N}\sum_{i=1}^{m}\sum_{t=1}^{n}\En\nrm*{\grad_{t}^{i}}_{q}^{2}
+ \frac{\breg(w^{\star}\dmid{}\bar{w})}{\eta{}N} 
+ O\prn*{\frac{B^{2}m}{\eta{}N}\prn*{\frac{1}{s}}^{\frac{p-1}{2}}
}.
\end{align*}
Let $\wt{w}$ denote $w_{t}^{i}$ for the index $(i,t)$ selected
uniformly at random in the final line of
\pref{alg:mirror_slow}. Interpreting the left-hand-side of this
expression as a conditional expectation over $\wt{w}$, we get
\begin{equation}
\label{eq:md_slow_wt}
\En\brk*{\poprisk(\wt{w})}-\poprisk(w^{\star}) \leq{}   \frac{\eta{}C_{q}}{2N}\sum_{i=1}^{m}\sum_{t=1}^{n}\En\nrm*{\grad_{t}^{i}}_{q}^{2}
+ \frac{\breg(w^{\star}\dmid{}\bar{w})}{\eta{}N} 
+ O\prn*{\frac{B^{2}m}{\eta{}N}\prn*{\frac{1}{s}}^{\frac{p-1}{2}}
}.
\end{equation}
Note that our boundedness assumptions imply
$\nrm*{\grad_{t}^{i}}_{q}^{2}\leq{}R^{2}$ and
$\breg(w^{\star}\dmid{}\bar{w})=\breg(w^{\star}\dmid{}0)\leq{}\frac{B^{2}}{2}$,
so when $s=\Omega(m^{\frac{2}{p-1}})$ this is bounded by
\[
\En\brk*{\poprisk(\wt{w})}-\poprisk(w^{\star}) \leq{}   \frac{\eta{}C_{q}R^{2}}{2}
+ O\prn*{\frac{B^{2}}{\eta{}N} 
} \leq O(\sqrt{C_qB^{2}R^{2}/N}),
\]
where the second inequality uses the choice of learning rate.

From here we split into two cases. In the general loss case, since
$\poprisk$ is $R$-Lipschitz with respect to $\ls_{p}$ (implied by the
assumption that subgradients lie in $\cX_q$ via duality), we get
\[
\poprisk(\wh{w})-\poprisk(w^{\star}) \leq{}\poprisk(\wt{w})-\poprisk(w^{\star}) + R\nrm*{\wh{w}-\wt{w}}_{p}.
\]
We now invoke \pref{thm:maurey_lp} once more, which implies that
\[
\En\nrm*{\wh{w}-\wt{w}}_{p} 
\leq O\prn*{
B\prn*{\frac{1}{s_{0}}}^{1-\frac{1}{p}}
}.
\]
We see that it suffices to take
$s_{0}=\Omega((N/C_q)^{\frac{p}{2(p-1)}})$ to ensure that this error
term is of the same order as the original excess risk bound.

In the linear model case, \pref{lem:maurey_prelim} directly implies
that
\[
\En\poprisk(\wh{w})
\leq{}\poprisk(\wt{w}) + O(\sqrt{B^{2}R^{2}/s_0}),
\]
and so $s_0=\Omega(N/C_q)$ suffices.

\end{proof}

\begin{proof}[\pfref{thm:mirror_linear_slow_lstar}] 
We begin from \pref{eq:md_slow_wt} in the proof of
\pref{thm:stochastic_mirror_slow} which, once $s=\Omega(m^{\frac{2}{p-1}})$,
implies 
  \begin{align*}
\En\brk*{\poprisk(\wt{w})}-\poprisk(w^{\star}) \leq{}   \frac{\eta{}C_{q}}{2N}\sum_{i=1}^{m}\sum_{t=1}^{n}\En\nrm*{\grad_{t}^{i}}_{q}^{2}
+ O\prn*{\frac{B^{2}}{\eta{}N}},
  \end{align*}
where $\wt{w}$ is the iterate $w_{t}^{i}$ selected uniformly at random at the final step
and the expectation is over
all randomness except the final sparsification step. Since the loss $\ls(\cdot,z)$ is
smooth, convex, and non-negative, we can appeal to Lemma 3.1 from
\citet{srebro2010smoothness}, which implies that 
\[
\nrm*{\grad_{t}^{i}}_{q}^{2}
= \nrm*{\grad\ls(w_{t}^{i},z_{t}^{i})}_{q}^{2}\leq{} 4\beta_q{}\ls(w_{t}^{i},z_{t}^{i}).
\]
Using this bound we have
\[
\En\brk*{\poprisk(\wt{w})}-\poprisk(w^{\star}) \leq{}   \frac{4\eta{}C_{q}\beta_{q}}{2N}\sum_{i=1}^{m}\sum_{t=1}^{n}\En\ls(w_{t}^{i},z_{t}^{i})
+ O\prn*{\frac{B^{2}}{\eta{}N}}
= 2\eta{}C_{q}\beta_{q}\cdot{}\En\brk*{\poprisk(\wt{w})}
+ O\prn*{\frac{B^{2}}{\eta{}N}}.
\]
Let $\veps\ldef{}2\eta{}C_q\beta_q$. Rearranging, we write
\[
(1-\veps)\En\brk*{\poprisk(\wt{w})}-\poprisk(w^{\star})
\leq
O\prn*{\frac{B^{2}}{2\eta{}N}}.
\]
When $\veps<1/2$, this implies $\En\brk*{\poprisk(\wt{w})}-(1+2\veps)\poprisk(w^{\star})
\leq
O\prn*{\frac{B^{2}}{2\eta{}N}}$,
and so, by rearranging, 
\[
\En\brk*{\poprisk(\wt{w})}-\poprisk(w^{\star})
\leq
O\prn*{\eta{}C_{q}\beta_qL^{\star} +  \frac{B^{2}}{2\eta{}N}}.
\]
The choice
$\eta=\sqrt{\frac{B^{2}}{C_{q}\beta_qL^{\star}N}}\wedge{}\frac{1}{4C_q\beta_q}$
ensures that $\veps\leq{}1/2$, and that
\[
\eta{}C_{q}\beta_qL^{\star} +  \frac{B^{2}}{2\eta{}N} = O\prn*{
\sqrt{\frac{C_q \beta_qB^{2}L^{\star}}{N}} + \frac{C_q\beta_q{}B^{2}}{N}
}.
\]
Now, \pref{lem:maurey_smooth} implies that, conditioned on $\wt{w}$,
we have
$\En\poprisk(\walg) \leq{} \poprisk(\wt{w}) + \frac{\beta_qB^{2}}{s_0}$.
The choice $s_0 =
\sqrt{\frac{\beta_q{}B^2N}{C_qL^{\star}}}\wedge\frac{N}{C_q}$
guarantees that this approximation term is on the same order as the excess risk bound of $\wt{w}$.
\end{proof}

\begin{proposition}
\label{prop:mirror_slow_initial}
Suppose we run \pref{alg:mirror_slow} with initial point $\bar{w}$
that is chosen by some randomized procedure independent of the data or
randomness used by \pref{alg:mirror_slow}. Suppose that we are promised that
this selection procedure satisfies $\En\nrm*{\bar{w}-w^{\star}}_{p}^{2}\leq{}\bar{B}^{2}$. Suppose that subgradients belong to $\cX_{q}$ for $q\geq{}2$, and
that $\cW\subseteq{}\cW_1$. Then, using learning rate
$\eta\ldef\frac{\bar{B}}{R}\sqrt{\frac{1}{C_{q}N}}$,  
$s=\Omega\prn*{m^{2(q-1)}\prn*{\nicefrac{B}{\bar{B}}}^{4(q-1)}}$,
 and 
$s_{0}=\Omega(\prn*{\nicefrac{N}{C_q}}^{\frac{q}{2}}\cdot\prn*{\nicefrac{B}{\bar{B}}}^{q})$, 
the algorithm guarantees
\[
\En\brk*{\poprisk(\wh{w})} - \poprisk(w^{\star}) \leq{} O\prn*{
\bar{B}R\sqrt{\frac{C_{q}}{N}}
}.
\]
\end{proposition}
\begin{proof}[\pfref{prop:mirror_slow_initial}]
We proceed exactly as in the proof of
\pref{thm:stochastic_mirror_slow}, which establishes that conditioned
on $\bar{w}$,
\[
\En\brk*{
\poprisk(\wh{w})} - \poprisk(w^{\star})
\leq{}   \frac{\eta{}C_{q}}{2N}\sum_{i=1}^{m}\sum_{t=1}^{n}\En\nrm*{\grad_{t}^{i}}_{q}^{2}
+ \frac{\breg(w^{\star}\dmid{}\bar{w})}{\eta{}N} 
+ O\prn*{\frac{B^{2}m}{\eta{}N}\prn*{\frac{1}{s}}^{\frac{p-1}{2}}
} + O\prn*{BR\prn*{\frac{1}{s_0}}^{1-1/p}}.
\]
We now take the expectation over $\bar{w}$. We have that
$\En\breg(w^{\star}\dmid\bar{w})=\frac{1}{2}\En\nrm*{\bar{w}-w^{\star}}_{p}^{2}\leq{}\bar{B}^{2}/2$. It
is straightforward to verify from here that the prescribed sparsity levels and learning rate give the desired bound.
\end{proof}

\begin{proof}[\pfref{thm:lipschitz_fast}]
Let $\wh{w}_{0}=0$, and let us use the shorthand $\gamma\ldef\gamma_q$.

We will show inductively that $\En\nrm*{\wh{w}_{k}-w^{\star}}_{p}^{2}\leq{}2^{-k}B^{2}\rdef{}B_k^{2}$. Clearly this is true for $\wh{w}_{0}$. Now assume the statement is true for $\wh{w}_k$. Then, since $\En\nrm*{\wh{w}_{k}-w^{\star}}_{p}^{2}\leq{}B_k^{2}$, \pref{prop:mirror_slow_initial} guarantees that
\[
\En\brk*{\poprisk(\wh{w}_{k+1})} - \poprisk(w^{\star}) \leq{} 
c\cdot{}B_{k}R\sqrt{\frac{C_{q}}{N_{k+1}}},
\]
where $c>0$ is some absolute constant. Since the objective satisfies the restricted strong convexity condition (\pref{ass:rsc}), and since $\poprisk$ is convex and $\cW$ is also convex, we have $\tri*{\grad{}\poprisk(w^{\star}),w-w^{\star}}\geq{}0$ and so
\[
\En\nrm*{\wh{w}_{k+1}-w^{\star}}_{p}^{2} \leq{} 
\frac{2c\cdot{}B_{k}R}{\gamma}\sqrt{\frac{C_{q}}{N_{k+1}}}.
\]
Consequently, choosing $N_{k+1}=C_{q}\cdot\prn*{\frac{4cR}{\gamma{}B_k}}^{2}$ guarantees that
\[
\En\nrm*{\wh{w}_{k+1}-w^{\star}}_{p}^{2} \leq{} \frac{1}{2}B_{k}^{2},
\]
so the recurrence indeed holds. In particular, this implies that
\[
\En\brk*{\poprisk(\wh{w}_{T})}-\poprisk(w^{\star}) \leq{} \frac{\gamma}{4}B_{T-1}^{2}=2^{-T}\cdot\frac{\gamma{}B^{2}}{2}.
\]
The definition of $T$ implies that
\[
T\geq{}\log_{2}\prn*{
\frac{N}{32C_{q}}\prn*{\frac{\gamma{}B}{Rc}}^{2}
},
\]
and so
\[
\En\brk*{\poprisk(\wh{w}_{T})}-\poprisk(w^{\star}) \leq{} 2^{-T}\cdot\frac{\gamma{}B^{2}}{2}
\leq{}O\prn*{
\frac{C_{q}R^{2}}{\gamma{}N}
}.
\]
This proves the optimization guarantee.

To prove the communication guarantee, let $m_k$ denote the number of consecutive machines used at
round $k$. The total number of bits broadcasted---summing the sparsity levels from \pref{prop:mirror_slow_initial} over $T$ rounds---is at most
\begin{align*}
&\log{}d\cdot{}\sum_{k=1}^{T}\prn*{m_k}^{2q-1}\prn*{\frac{B}{B_{k-1}}}^{4(q-1)}+
\prn*{\frac{N_{k}}{C_q}}^{\frac{q}{2}}\cdot\prn*{\frac{B}{B_{k-1}}}^{q},
\end{align*}
plus an additive $O(m\log(BRN))$ term to send the scalar norm for each
sparsified iterate $\wh{w}_{i}$.
Note that we have $m_k=\frac{N_k}{n}\vee{}1$, so this is at most
\begin{align*}
&\log{}d\cdot{}\sum_{k=1}^{T}\prn*{\frac{N_k}{n}}^{2q-1}\prn*{\frac{B}{B_{k-1}}}^{4(q-1)}+
\prn*{\frac{N_{k}}{C_q}}^{\frac{q}{2}}\cdot\prn*{\frac{B}{B_{k-1}}}^{q}.
\end{align*}
The first term in this sum simplifies to  $O\prn*{\log{}d\cdot{}
\prn*{\frac{C_qR^{2}}{n\gamma^{2}B^{2}}}^{2q-1}
}\cdot\sum_{k=1}^{T}2^{(4q-3)k}$, while the second simplifies to
$O\prn*{\log{}d\cdot\prn*{\frac{R}{\gamma{}B}}^{q}2^{q}}\cdot{}\sum_{k=1}^{T}2^{qk}$. We
use that $\sum_{t=1}^{T}\beta^{t}\leq{}\beta^{T+1}$ for $\beta\geq{}2$
to upper bound by 
\begin{align*}\textstyle
& O\biggl(\log{}d\cdot{}
\prn*{\frac{C_qR^{2}}{n\gamma^{2}B^{2}}}^{2q-1}2^{q}
\biggr)\cdot2^{(4q-3)T}+
O\prn*{\log{}d\prn*{\frac{R}{\gamma{}B}}^{q}2^{q}}\cdot{}2^{qT}.
\end{align*}
Substituting in the value of $T$ and simplifying leads to a final
bound of
\begin{align*}
& O\biggl(\log{}d\cdot{}
\prn*{\frac{\gamma^{2}B^{2}}{C_{q}R^{2}}}^{2(q-1)}m^{2q-1}N^{2(q-1)}
+
\log{}d\cdot{}\prn*{\frac{\gamma{}BN}{C_q{}R}}^{q}
\biggr).\stepcounter{equation}\tag{\theequation}\label{eq:fast_rate_final_comm}
\end{align*}
\end{proof}

\begin{proof}[\pfref{prop:sparse_regression}]
It immediately follows from the definitions in the proposition that \pref{alg:mirror_fast} guarantees
\[
\En\brk*{\poprisk(\wh{w}_{T})}-\poprisk(w^{\star})\leq{} O\prn*{
\frac{C_{q}B^{2}R^{2}}{\gamma_qN}
},
\]
where $\gamma_q$ is as in \pref{ass:rsc}. We now relate $\gamma_q$ and $\gamma$. From the optimality of $w^{\star}$ and strong convexity of the square loss with respect to predictions it holds that for all $w\in\cW_p$,
\[
\En\brk*{\poprisk(w)}-\poprisk(w^{\star}) - \tri*{\grad{}\poprisk(w^{\star}),w-w^{\star}} \geq{} \En\tri*{x,w-w^{\star}}^{2}.
\]
Our assumption on $\gamma$ implies
\[
\En\tri*{x,w-w^{\star}}^{2} = \nrm*{\Sigma^{1/2}(w-w^{\star})}_{2}^{2} \geq{} \gamma\nrm*{w-w^{\star}}_{2}^{2}.
\]
Using \pref{prop:l1_cone}, we have
\[
\nrm*{w-w^{\star}}_{p}\leq{}\nrm*{w-w^{\star}}_{1}
\leq{}2\nrm*{(w-w^{\star})_{S}}_{1}
\leq{}2\sqrt{k}\nrm*{(w-w^{\star})_{S}}_{2}
\leq{}2\sqrt{k}\nrm*{w-w^{\star}}_{2}
\]
Thus, it suffices to take $\gamma_q=\frac{\gamma}{4k}$.

\end{proof}
The following proposition is a standard result in high-dimensional
statistics. For a given vector $w\in\bbR^{d}$, let $w_{S}\in\bbR^{d}$
denote the same vector with all coordinates outside $S\subseteq\brk*{d}$ set to zero.
\begin{proposition}
\label{prop:l1_cone}
Let $\cW$, $w^{\star}$, and $S$ be as in \pref{prop:sparse_regression}. All $w\in\cW$ satisfy the inequality $\nrm*{(w-w^{\star})_{S^{\comp}}}_{1}\leq{}\nrm*{(w-w^{\star})_{S}}_{1}$.
\end{proposition}
\begin{proof}[\pfref{prop:l1_cone}]
Let $\nu=w-w^{\star}$. From the definition of $\cW$, we have that for all $w\in\cW$,
\begin{align*}
\nrm*{w^{\star}}_{1}\geq{}\nrm*{w}_{1} &= \nrm*{w^{\star}+\nu}_{1}.
\end{align*}
Applying triangle inequality and using that the $\ls_1$ norm
decomposes coordinate-wise:
\[
\nrm*{w^{\star}+\nu}_{1} = \nrm*{w^{\star}+\nu_{S} + \nu_{S^{\comp}}}_{1}
= \nrm*{w^{\star}+\nu_{S}}_{1} + \nrm*{\nu_{S^{\comp}}}_{1} \geq{} \nrm*{w^{\star}}_{1}-\nrm*{\nu_{S}}_{1} + \nrm*{\nu_{S^{\comp}}}_{1}.
\]
Rearranging, we get $\nrm*{\nu_{S^{\comp}}}_{1}\leq{}\nrm*{\nu_{S}}_{1}$.
\end{proof}

\begin{proof}[\pfref{prop:schatten}]
To begin, we recall from \citet{kakade2012regularization} that the
regularizer $\cR(W)=\frac{1}{2}\nrm*{W}_{S_p}^{2}$ is $(p-1)$-strongly
convex for $p\leq{}2$. This is enough to show under our assumptions
that the centralized version of mirror descent (without
sparsification) guarantees excess risk
$O\prn*{\sqrt{\frac{C_qB_1^{2}R_{q}^{2}}{N}}}$, with $C_q=q-1$, which matches the $\ls_1/\ls_q$ setting.

What remains is to show that the new form of sparsification indeed
preserves Bregman divergences as in the $\ls_1/\ls_q$ setting. We now show that when $W$ and $W^{\star}$ have $\nrm*{W}_{S_1}\vee\nrm*{W^{\star}}_{S_1}\leq{}B$, 
\[
\En\brk*{
\breg(W^{\star}\dmid{}\maurey(W)) - \breg(W^{\star}\dmid{}W)
} \leq{} O\prn*{B^{2}\prn*{\frac{1}{s}}^{\frac{p-1}{2}}}.
\]
To begin, let $U\in\bbR^{d\times{}d}$ be the left singular vectors of $W$ and $V\in\bbR^{d\times{}d}$ be the right singular vectors. We define $\wh{\sigma}=\frac{\nrm*{W}_{S_1}}{s}\sum_{\tau=1}^{s}e_{i_\tau}$, so that we can write $W=U\diag(\sigma)V^{\trn}$ and $\maurey(W)=U\diag(\wh{\sigma})V^{\trn}$.

Now note that since the Schatten norms are unitarily invariant, we have
\[
\nrm*{W-\maurey(W)}_{S_{p}} = \nrm*{U\diag(\sigma-\sigmah)V^{\trn}}_{S_{p}}
= \nrm*{\sigma-\sigmah}_{p}
\]
for any $p$. Note that our assumptions imply that
$\nrm*{\sigma}_{1}\leq{}B$, and that $\sigmah$ is simply the vector
Maurey operator applied to $\sigma$, so it follows immediately from \pref{thm:maurey_lp} that
\begin{equation}
\label{eq:schatten_maurey}
\En\nrm*{\sigma - \wh{\sigma}}_{p} \leq{} 4B\prn*{\frac{1}{s}}^{1-1/p}\quad\text{and}\quad \sqrt{\En\nrm*{\sigma-\wh{\sigma}}_{\infty}^{2}} \leq{} 4B\prn*{\frac{1}{s}}^{1/2}.
\end{equation}

Returning to the Bregman divergence, we write
\begin{align*}
\breg(W^{\star}\dmid{}\maurey(W)) - \breg(W^{\star}\dmid{}W) 
&= D_{\cR}(W\dmid{}\maurey(W)) + \tri*{\grad\cR(\maurey(W))-\grad\cR(W),W-W^{\star}} \\
&\leq D_{\cR}(W\dmid{}\maurey(W)) + \nrm*{\grad\cR(\maurey(W))-\grad\cR(W)}_{S_{\infty}}\nrm*{W-W^{\star}}_{S_1} \\
&\leq D_{\cR}(W\dmid{}\maurey(W)) + 2B\nrm*{\grad\cR(\maurey(W))-\grad\cR(W)}_{S_{\infty}}.
\end{align*}
It follows immediately using \pref{lem:bregman_norm_bound} that
\[
\breg(W\dmid{}\maurey(W))\leq{}3B\nrm*{W-\maurey(W)}_{S_p}=3B\nrm*{\sigma-\wh{\sigma}}_{p}.
\]
To make progress from here we use a useful representation for the
gradient of $\cR$. Define 
\[
g(\sigma) = \nrm*{\sigma}_{p}^{2-p}\cdot\prn*{
\abs*{\sigma_1}^{p-1}\sgn(\sigma_1),\ldots,\abs*{\sigma_d}^{p-1}\sgn(\sigma_d)
}.
\]
Then using Theorem 30 from \citet{kakade2012regularization} along with \pref{eq:lp_grad}, we have
\[
\grad\cR(W) = U\diag(g(\sigma))V^{\trn},\quad\text{and}\quad\grad\cR(\maurey(W)) = U\diag(g(\sigmah))V^{\trn}.
\]
For the gradient error term, unitary invariance again implies that
\[
\nrm*{\grad\cR(\maurey(W))-\grad\cR(W)}_{S_{\infty}} = \nrm*{U\diag(g(\sigma)-g(\sigmah))V^{\trn}}_{S_{\infty}}
=\nrm*{g(\sigma)-g(\sigmah)}_{\infty}.
\]
\pref{lem:bregman_gradient_holder} states that
\[
\nrm*{g(\sigma)-g(\sigmah)}_{\infty}
\leq{}2B^{2-p}\nrm*{\sigma-\sigmah}_{\infty}^{p-1} +
\nrm*{\sigma-\sigmah}_{p}.
\]
Putting everything together, we get
\[
\breg(W^{\star}\dmid{}\maurey(W)) - \breg(W^{\star}\dmid{}W)  \leq{} 
5B\nrm*{\sigma-\wh{\sigma}}_{p} + 4B^{3-p}\nrm*{\sigma-\wh{\sigma}}_{\infty}^{p-1}.
\]
The desired result follows by plugging in the bounds in \pref{eq:schatten_maurey}.
\end{proof}


\end{document}